%% file: main.tex
\documentclass[11pt]{article}

\input{packages}

\theoremstyle{plain}
\newtheorem{theorem}{Theorem}[section]
\newtheorem{proposition}[theorem]{Proposition}
\newtheorem{lemma}[theorem]{Lemma}
\newtheorem{corollary}[theorem]{Corollary}
\theoremstyle{definition}
\newtheorem{definition}[theorem]{Definition}
\newtheorem{assumption}[theorem]{Assumption}
\theoremstyle{remark}

\input{notations}

\begin{document}

\title{Provable Reinforcement Learning with a Short-Term Memory}

\author{%
Yonathan Efroni \thanks{Microsoft Research, NYC. Email: \texttt{jonathan.efroni@gmail.com}}
\and
 Chi Jin  \thanks{Princeton University. Email: \texttt{chij@princeton.edu}}
 \and
 Akshay Krishnamurthy \thanks{Microsoft Research, NYC. Email: \texttt{akshaykr@microsoft.com}}
 \and
 Sobhan Miryoosefi  \thanks{Princeton University. Email: \texttt{miryoosefi@cs.princeton.edu}}
}

\date{}
\maketitle

\begin{abstract}
\input{sections/abstract}

\end{abstract}

\section{Introduction}
\label{sec:intro}
\input{sections/intro}

\subsection{Related Work}
\label{sec:related_work}
\input{sections/related}

\section{Preliminaries}
\label{sec:prelim}
\input{sections/prelim}

\section{Warmup: Tabular Case}
\label{sec:megastate}
\input{sections/naivesolution}


\section{Main results}
\label{sec:golf}
\input{sections/golf}

\subsection{Linear $m$-step Decodable POMDP}
\label{sec:linear}
\input{sections/linear}

\section{Challenges and Proof Overview}
\label{sec:proof_overview}
\input{sections/proof_overview}

\section{Conclusion}
\label{sec:conclusion}
\input{sections/conclusion}

\bibliography{ref.bib}
\bibliographystyle{unsrtnat}

\newpage
\appendix

\input{appendix/appendix_main}

\end{document}

%% file: packages.tex
\usepackage{fullpage}
\usepackage[round]{natbib}
\usepackage{algorithm}
\usepackage[noend]{algorithmic}
\usepackage{amsmath,amsthm,amsfonts,amssymb}
\usepackage{mathtools}
\usepackage{amsmath}
\usepackage{hyperref}
\usepackage{color}
\usepackage{mathrsfs}
\usepackage{enumitem}
\usepackage{bm}
\usepackage{multirow}
\usepackage{booktabs}
\usepackage{makecell}
\usepackage{graphicx}
\usepackage{microtype}
\usepackage{subfigure}
\usepackage{caption}
\usepackage{thmtools}
\usepackage{thm-restate}
\usepackage{setspace}
\usepackage{makecell}
\definecolor{light-gray}{gray}{0.85}
\usepackage{colortbl}
\usepackage{xcolor}
\usepackage{hhline}
\usepackage{amsfonts}

\usepackage{times}

\usepackage[textsize=tiny]{todonotes}

\usepackage[capitalize,noabbrev]{cleveref}

\usetikzlibrary{fit}

%% file: sections/abstract.tex

Real-world sequential decision making problems commonly involve partial observability, which requires the agent to maintain a memory of history in order to infer the latent states, plan and make good decisions. Coping with partial observability in general is extremely challenging, as a number of worst-case statistical and computational barriers are known in learning Partially Observable Markov Decision Processes (POMDPs). Motivated by the problem structure in several physical applications, as well as a commonly used technique known as ``frame stacking'', this paper proposes to study a new subclass of POMDPs, whose \emph{latent states can be decoded by the most recent history of a short length $m$}. We establish a set of upper and lower bounds on the sample complexity for learning near-optimal policies for this class of problems in both tabular and rich-observation settings (where the number of observations is enormous). In particular, in the rich-observation setting, we develop new algorithms using a novel ``moment matching'' approach with a sample complexity that scales exponentially with the short length $m$ rather than the problem horizon, and is independent of the number of observations. Our results show that a short-term memory suffices for reinforcement learning in these environments.

%% file: sections/intro.tex

Reinforcement learning is a well-studied paradigm for sequential
decision making, in which an agent learns to make
decisions in a stateful environment to accumulate reward. The most
common framework for reinforcement learning---particularly for
theoretical analysis---is the Markov Decision Process (MDP), in
which the environment is summarized by a state that is observable
to the agent. 
One notable feature of the MDP is that the agent can be
\emph{memoryless}, meaning that it need not 
remember past states to make decisions in the present.  However, many
real world problems exhibit partial observability and require the
agent to maintain a memory of the past to infer the latent states, plan, and make good decisions.  These
problems are best modeled via the framework of Partially Observable
MDPs (POMDPs).

As a motivating example, consider a control task of navigating a robot that
perceives the environment through a visual system like a first-person
camera. 
Here, a single image may identify the agent's location, but it would
not identify the agent's velocity, which is necessary for deciding how much force should be applied in order to
accelerate or brake. For optimal control, the agent would
have to maintain a memory of past images and infer its velocity from
this historical information. This problem can be modeled as a POMDP
where the system state is the position and velocity of the
agent. However, the state cannot be inferred using a single image,
hence it is \emph{partially observable}.




Maintaining a memory and reasoning over histories in POMDPs is notoriously
challenging, as evidenced by a number of complexity-theoretic barriers: computing the optimal policy (or planning) is computationally intractable~\citep{papadimitriou1987complexity} and learning an unknown POMDP
incurs a sample complexity that scales exponentially with the
horizon~\citep{mossel2005learning,jin2020sample}. These lower bounds 
often involve constructions that require the agent to reason over very long
histories. 
However, they are worst-case in nature, so they  
leave open the possibility of obtaining positive results for subclasses of POMDPs with special structure of practical interest.

One such structure concerns applications of POMDPs where the agent only needs a \emph{short-term memory}. 
This structure holds in our motivating example, since the velocity can be recovered from just the most recent images. 
Short-term memory is also frequently used in the design of practical
algorithms, which concatenate observations from
the most recent time steps and use them to make decisions---a
technique called
``frame-stacking''~\citep{mnih2013playing,mnih2015human,hessel2018rainbow}. This
gives rise to a natural question: \emph{Can we develop a theoretical
framework and design provably efficient algorithms for reinforcement learning
with a short-term memory?}


\paragraph{Our contributions.}
In this paper, we address the question above by proposing a new class of models---$m$-step decodable POMDPs. This class is a subclass of
general POMDPs where the latent state can be determined by the observations and actions of the $m$ most recent time steps via an \emph{unknown} decoding function $\phi^\star$ (see Assumption \ref{asm:nstep_decode}). 

As a warm-up example, we first consider the tabular
setting, where the number of states, observations, and actions, are
all relatively small. Here a simple technique which stacks the observations and actions in the $m$ most recent steps into a new ``mega''-states yields an algorithm with sample complexity
$\mathcal{O}(H(OA)^m) $  where $O,A$ are the number of
observations and actions respectively and $H$ is the episode
length. We also show an $\Omega(A^m)$ lower bound, establishing that
an exponential dependence on $m$ is indeed necessary.

Our main result concerns the rich-observation setting where the
observation space can be arbitrarily complex ($O$ is arbitrarily large) and one must use function
approximation for generalization. We present a clean solution to this problem with a simple variant of the \golf~algorithm \citep{jin2021bellman}, which was originally proposed for RL with general function approximation in the observable/Markovian setting. We show that our algorithm finds a near-optimal policy within {$\mathcal{O}(\textrm{poly}(H)  A^m S \cdot \log |\mathcal{F}|)$} samples, where $S$ is the number of latent states and $|\mathcal{F}|$ is cardinality of the function class. Most importantly, our sample complexity does not depend on the number of observations $O$. 
We further extend our result to the setting where the latent dynamics correspond to a linear MDP, with $S$ in the sample complexity replaced by latent dimension $d$.


Our results in the rich observation setting crucially rely on a novel
concept that we call the ``moment matching policy,'' which breaks
historical dependencies while matching the joint distribution of
states, observations, and actions for a short time interval (See
Section \ref{sec:moment_matching}).  These policies enable a low-rank
or bilinear decomposition of the Bellman error of any value function
in the POMDP, which is essential for obtaining sample efficient
results in the rich observation
setting~\citep{jiang2017contextual,jin2021bellman,du2021bilinear}. As
such, the moment matching policies might be of independent interest
for future research in partial observability.

%% file: sections/related.tex
Partial observability is a central challenge in practical
reinforcement learning settings and, as such, it has been the focus of
a large body of empirical work. The two most popular high-level
approaches are to use recurrent or other ``temporally extended''
neural
architectures~\citep{hausknecht2015deep,zhu2017improving,igl2018deep,hafner2019dream},
or to employ feature engineering~\citep{mccallum1993overcoming}, for
example by providing the most recent observations as input to the
agent~\citep{mnih2013playing,mnih2015human,hessel2018rainbow}. 
However, we are not aware of any theoretical treatment of these
methods in the RL context.

Turning to theoretical results, two lines of work are related to our
own. The first addresses RL with partial
observability.~\citet{kearns1999approximate,kearns2002sparse,evendar2005reinforcement}
provide sparse sampling techniques that attain $A^H$-type sample
complexity for various POMDP tasks, including without resets. These
bounds have an undesirable exponential dependence on the horizon,
which we show can be removed in some special cases. A more recent line
of
work~\citep{azizzadenesheli2016reinforcement,guo2016pac,jin2020sample}
use method of moment estimators (based on spectral methods for
learning latent variable models~\citep[c.f.,][]{anandkumar2014tensor}
to obtain guarantees in \emph{undercomplete} tabular POMDPs. However,
undercompleteness, which means that the emission matrix is robustly
rank $|O|$, need not hold in our setting, so these results are
orthogonal to ours.

The second line of work concerns rich observation RL, where the
observation space can be infinite and arbitrarily complex, in (for the
most part) \emph{Markovian} environments. These works provide
structural conditions that permit sample efficient RL with function
approximation~\cite{jiang2017contextual,sun2019model,jin2021bellman,du2021bilinear,foster2021statistical}
 as well as algorithms that are
provably efficient in some special
cases~\citep{du2019provably,misra2020kinematic,agarwal2020flambe,uehara2021representation}.
However, as we will see, these structural conditions are not satisfied
in our POMDP model so these results do not directly apply.

Outside of RL settings, the use of memory is prevalent in controls and
time series
prediction~\citep{ljung1998system,box2015time,hamilton1994time},
dating back to the seminal work of~\citet{kalman1960new}. Short-term
memory is explicit in several autoregressive models, such as the AR
and ARMA models. It is also classical to leverage memory in many
control-theoretic settings. More recently, short-term memory has been
employed in control settings, where one can use stability arguments to
show that a short memory window suffices to approximate the optimal
policy~\citep{verhaegen1993subspace,arora2018towards,agarwal2019online,oymak2019non,simchowitz2019learning}. These
ideas provide further motivation for our study but the techniques
developed in these continuous settings do not seem useful for discrete
RL problems where exploration is challenging.

%% file: sections/prelim.tex


\paragraph{Notation.} 
We use $[H]$ to denote the set $\{1,\ldots,H\}$.  For any indexed
sequence $a_1, a_2, \ldots$, we use $a_{i:j}$ to denote the
subsequence $( a_{\max\{1,i\}},\ldots,a_{\max\{1,j\}})$ for any $i,
j \in \mathbb{Z}$ with $i\le j$. 
We adopt the standard
big-oh notation and write $f = \tilde{\Ocal}(g)$ to denote that
$f=\Ocal(g \cdot \max\{1,\mathrm{polylog}(g)\})$.

\paragraph{POMDPs.} We consider an episodic Partially Observable Markov Decision Process (POMDP), which can be specified by $\Mcal = (\Scal, \Ocal,\Acal, H, \PP, \OO, r)$. Here $\Scal$ is the \emph{unobservable} state space, $\Ocal$ is the observation space, $\Acal$ is the action space, and $H$ is the horizon. $\PP=\{\PP_h\}_{h=1}^H$ is a collection of \emph{unknown} transition probabilities with $\PP_h(s' \mid s,a)$ equal to the probability of transitioning to $s'$ after taking action $a$ in state $s$ at the $h^{\text{th}}$ step. $\OO = \{\OO_h\}_{h=1}^H$ are the \emph{unknown} emissions with $\OO_h(o \mid s)$ equal to probability that the environment emits observation $o$ when in state $s$ at the $h^{\text{th}}$ step. $r = \{r_h: \Ocal \rightarrow [0,1]\}_{h=1}^H$ are the deterministic reward functions.\footnote {We study deterministic reward for simplicity. Our results readily generalize to random rewards.} Throughout the paper, we assume that $\sum_{h=1}^H r_h(o_h) \leq 1$ almost surely. We assume our action space is finite, $|\Acal| \leq A$, and in all sections except \cref{sec:linear}, we assume our state space is also finite, $|\Scal| \leq S$.

\begin{figure}
      \begin{center}
      \includegraphics[width=0.45\textwidth]{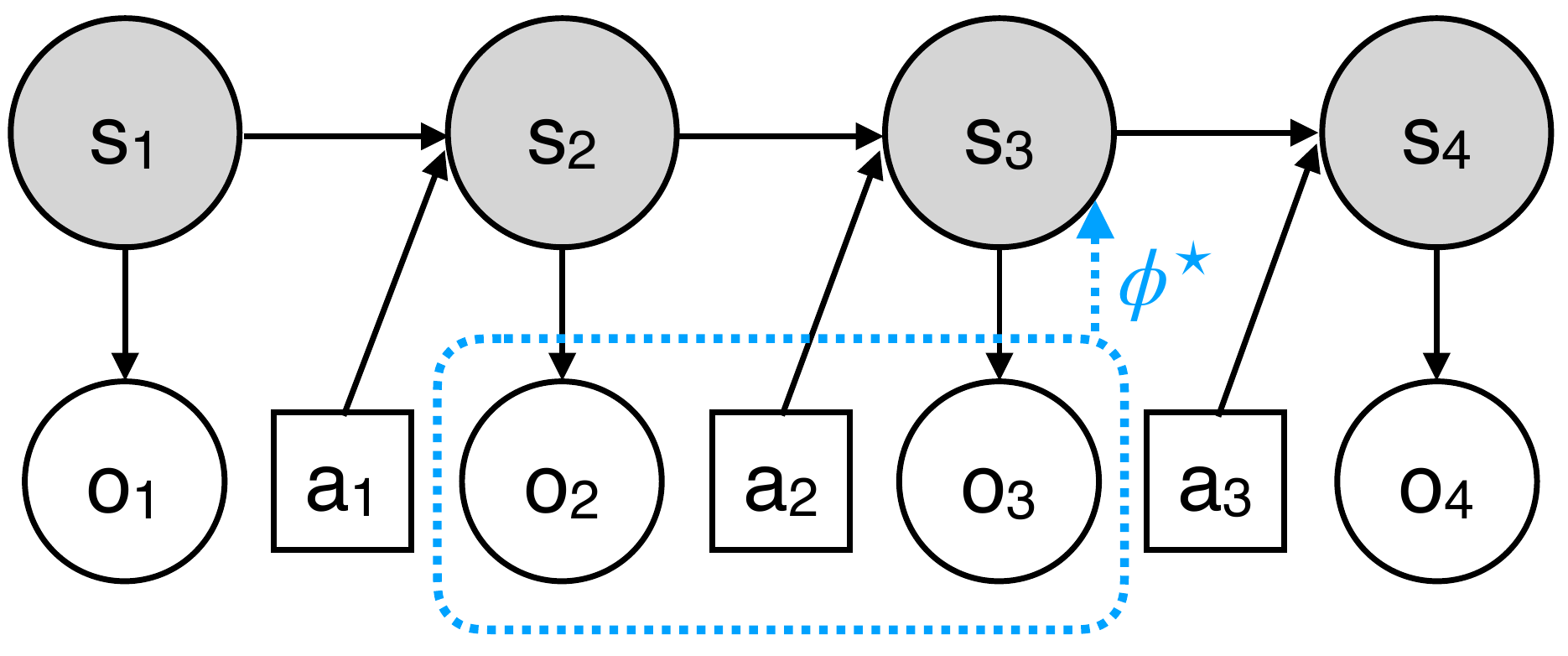}
      \caption{A schematic of a 2-step decodable POMDP. The latent state $s_h$ can be recovered using only $o_{h-1},a_{h-1},o_h$, so a short-term memory suffices for decision making.}
      \label{fig:decodable_fig}
      \end{center}
\end{figure}

\paragraph{Interaction protocol.} In a POMDP, the states are hidden and unobservable; i.e., the agent is only able to see the observations and its own actions. 
Each episode starts with initial state $s_1$ which is sampled from some \emph{unknown} initial distribution.
Then, at each step $h \in [H]$, the environment emits observation $o_h \sim \OO_h(\cdot \mid s_h)$, the agent observes $o_h \in \Ocal$, receives reward $r_h(o_h)$, and takes action $a_h \in \Acal$ causing the environment to transition to $s_{h+1} \sim \PP(\cdot \mid s_h,a_h)$.


\paragraph{Multi-step decodability.} We first define the notion of reachable trajectories.
\begin{definition}[Reachable trajectories] We say a trajectory $\tau = (s_1, o_1, a_1, r_1, s_2,\ldots,s_H,o_H,a_H,r_H)$ is \emph{reachable} if the probability $P((s,o)_{1:H}|a_{1:H}) = (\prod_{h=1}^H \OO(o_h|s_h))\cdot (\prod_{h=1}^{H-1} \PP(s_{h+1}|s_h, a_h))$ is strictly positive.
\end{definition}

Now we present the key structural assumption of this paper, which
assumes that a suffix of length $m$ of the history suffices to decode
the latent state.  We use $\Zcal_h$ to denote the set of suffixes at
step $h$, given by $\Zcal_h =
(\Ocal \times \Acal)^{\min\{h-1,m-1\}} \times \Ocal$.\footnote{When
$h \le m$, this suffix includes the entire history starting from time
step $1$.} Additionally, since it will appear frequently in subscripts
in the sequel, let $m(h) = \min\{h-m+1,1\}$.

\begin{assumption}[$m$-step decodability]
\label{asm:nstep_decode}
There exists an \emph{unknown} decoder $\phi^\star=\{\phi^\star_h: \Zcal_h \rightarrow \Scal \}_{h=1}^H$ such that for every \emph{reachable} trajectory $\tau = (s, o, a)_{1:H}$, we have $s_h = \phi^\star_h(z_h)$ for all $h \in [H]$, where $z_h = ((o, a)_{m(h):h-1},o_h)$.
\end{assumption}

We call a POMDP satisfying \Cref{asm:nstep_decode} an \textbf{$m$-step decodable POMDP}. 
An example with $m=2$ is illustrated in \Cref{fig:decodable_fig}.
Note that restricting decodability to only hold on reachable sequences results in a weaker assumption, which can include more practical settings.

Our model is a generalization of the block Markov decision process
(BMDP) \citep{jiang2017contextual,du2019provably}, which corresponds to the case where $m=1$. 
However, we emphasize that when $m=1$ there is no partial
observability since the current observation suffices for decoding the
hidden state. Thus the BMDP model does not require memory while, for
$m>1$, our model does.

\paragraph{Policies and value functions.}
For $m$-step decodable POMDPs, we consider the class of $m$-step policies. 
An $m$-step policy $\pi$ is a collection $\pi = \{ \pi_h: \Zcal_h \rightarrow \Acal \}$ that maps suffixes of length $m$ of the history to actions. 
The agent follows policy $\pi$ by choosing action $a_h = \pi_h(z_h)$ at the $h^{\textrm{th}}$ step, where $z_h = ((o,a)_{m(h):h-1},o_h) \in \Zcal_h$. 
We denote $V^\pi$ as the value for policy $\pi$, defined as the expected total reward obtained when following policy $\pi$, that is $V^\pi = \EE_\pi[\sum_{h=1}^H r_h(o_h)]$. 

We can similarly define the value at step $h$ to be the expected future reward when starting from step $h$. 
While this value may depend on the entire history in general, it is not hard to show that in $m$-step decodable POMDPs with an $m$-step policy $\pi$, this value only depends on the suffix of length $m$. 
Mathematically, we can define $V^\pi_h : \Zcal_h \rightarrow [0,1]$ to be the value function at step $h$ for (the $m$-step) policy $\pi$ as
\begin{equation*}
	V^\pi_h(z) \defeq E_\pi \big[ \sum_{h'=h+1}^H r_{h'}(o_{h'}) \mid z_h = z \big].
\end{equation*}
Similarly we define $Q^\pi_h : \Zcal_h \times \Acal \rightarrow [0,1]$ to be the $Q$-value function at step $h$ for (the $m$-step) policy $\pi$ as
\begin{equation*}
	Q^\pi_h(z,a) \defeq E_\pi \big[ \sum_{h'=h+1}^H r_{h'}(o_{h'}) \mid z_h = z, a_h = a \big].
\end{equation*}

Furthermore, \cref{asm:nstep_decode} guarantees that there exists an $m$-step policy $\pi^\star$ which is optimal in the sense $V^{\pi^\star} = \max_{\pi \in \Pi}V^{\pi}$ where $\Pi$ is the class of all policies, which may depend on the entire history. We use $V^\star$, $V^\star_h$, and $Q^\star_h$ to denote $V^{\pi^\star}$, $V^{\pi^\star}_h$, and $Q^{\pi^\star}_h$ respectively.

We define the \emph{Bellman operator} $\Tcal_h$ at step $h$ as
\begin{align*}
	(\Tcal_h g)(z,a) \defeq \EE\big[r_{h+1}(o_{h+1})+ \max_{a_{h+1} \in \Acal} g(z_{h+1},a_{h+1})  
	\mid z_h = z, a_h = a \big],
\end{align*}
for any function $g: \Zcal_{h+1} \times \Acal \rightarrow [0,1]$ that depends on $m$-step suffix. It is not hard to check that $Q^\star$ satisfies the Bellman optimality equation $Q^\star_h(z,a) = (\Tcal_h Q^\star_{h+1})(z,a)$ for all $h \in [H]$ and $(z,a) \in \Zcal_h \times \Acal$. 

Finally, for two non-stationary policies $\pi_1,\pi_2$ we use the notation $\pi_1\circ_{t}\pi_2$ be a non-stationary policy that executes $\pi_1$ for $t-1$ time steps and then, starting from the $t^{\textrm{th}}$ time step, executes $\pi_2$.

\paragraph{Learning objective.} Our objective is to learn an $\epsilon$-optimal policy $\widehat{\pi}$, which satisfies $V^{\widehat{\pi}} \geq V^\star - \epsilon$.


\subsection{Function approximation}
In the function approximation setting, the learner is given a function
class $\Fcal = \Fcal_1 \times \dots \times \Fcal_H$, where
$\Fcal_h \subseteq (\Zcal_h \times \Acal \rightarrow [0,1])$ consists
of candidate functions to approximate $Q^\star_h$---the optimal
$Q$-value function at step $h$. Without loss of generality we assume
that $f_{H+1} \equiv 0$. We present two assumptions that are commonly
adopted in the literature to avoid challenges associated with
reinforcement learning with function approximation (e.g., the hardness
results in \citealt{krish2016hardness,weisz2021exponential}).

\begin{assumption}[Realizability]
\label{asm:realizability}
$Q^\star_h \in \Fcal_h$ for all $h \in [H]$.
\end{assumption}
This assumption requires that our function class $\Fcal$ in fact contains the the optimal $Q$-value function, $Q^\star$.

\begin{assumption}[Generalized Completeness]
\label{asm:generalized_completeness}	
$\Tcal_h f_{h+1} \in \Gcal_h$ for all  $h \in [H]$ and $f_{h+1} \in \Fcal_{h+1}$,
where $\Gcal = \Gcal_1 \times \dots \times \Gcal_H$ is an auxiliary function class provided to the learner, with $\Fcal_h \subseteq \Gcal_h \subseteq (\Zcal_h \times \Acal \rightarrow [0,1])$.
\end{assumption}
The \emph{generalized completeness} \citep{antos2008learning,chen2019info} assumption requires the auxiliary function class $\Gcal$ to be rich enough so that applying the Bellman operator on any function in the original class $\Fcal$ results in a function in $\Gcal$. 
If we choose $\Gcal = \Fcal$, \cref{asm:generalized_completeness}
reduces to the standard completeness assumption, but separating the
two classes provides more flexibility.


We use covering numbers to capture the statistical complexity, or effective size, of the classes $\Fcal$ and $\Gcal$.
\begin{definition}[$\epsilon$-cover]
The $\epsilon$-covering	of a set $\Xcal$ under a metric $\rho$, denoted by $\Ncal\rbr{\Xcal, \epsilon,\rho}$ is the minimum integer $n$ such that there exists a subset $\Xcal_0 \subseteq \Xcal$  with $\abr{\Xcal_0}=n$ and for any $x\in \Xcal$ there exists $y\in \Xcal_0$  such that $\rho(x,y)\leq \epsilon$. 
\end{definition}
In this work, for the function class $\Fcal = \Fcal_1 \times \dots \times \Fcal_H$, we use the metric $\rho(f^{(1)}- f^{(2)}) = \max_{h \in H} \norm{f_h^{(1)} - f_h^{(2)}}_{\infty}$ where $f^{(1)},f^{(2)}\in \Fcal$. Since this metric is fixed throughout the paper, we use a simpler notation of $\Ncal_{\Fcal}(\epsilon)$ to denote the $\epsilon$-covering number of $\Fcal$.

Finally, let $\pi_f = \{z_h \mapsto \arg\max_{a \in \Acal} f_h(z_h,a)\}^H_{h=1}$ denote the greedy policy with respect to $f\in \Fcal$, where ties are broken in a canonical fashion.


%% file: sections/naivesolution.tex

We start by considering a basic setting where the numbers of states, actions, and observations are all finite and small, so we additionally have $|\Ocal| \leq O$.
In this setting, we describe a simple reduction from an $m$-step decodable POMDP to a new MDP with augmented states.  With this reduction at hand, we can to apply any RL algorithms designed for the fully observable setting to learn a near optimal policy.

In the reduction to an MDP, instead of using only the current
observation $o_h$ as the state at time $h$, we use the $m$-length
suffix of observations and actions $z_h$. We refer to such a suffix as
a \emph{megastate}. Formally, the reduction uses a time-dependent
extended state space $\Scal^{m,h} = \Zcal_h$, and 
the next result establishes that $\Scal^{m,h}$ induces Markovian
dynamics. Additionally, an optimal policy of this MDP is also an
optimal policy of the original $m$-step decodable POMDP.\footnote{All
proofs are deferred to the appendices.}

\begin{proposition}[Megastate MDP]\label{prop: megastate MDP}
The state space $\Scal^{m,h}$ induces Markovian dynamics $\PP^m$ and
reward $r^m$. Let this MDP be $\Mcal^m = \rbr{\Scal^{m,h}, \Acal,
H,\PP,r}$. An optimal policy of $\Mcal^m $ is an optimal policy
of the $m$-step decodable POMDP.
\end{proposition}

We refer to $\Mcal^m$ as the megastate MDP. With this proposition, we
can apply any RL algorithm (e.g., UCB-VI by \citealt{azar2017minimax}) to
the megastate MDP to learn a near optimal policy for the original
POMDP. Since the cardinality of the state space of $\Mcal^m$ at each
step is $\max_{h \in [H]}\abr{\Scal^{m,h}} \leq O^mA^{m-1}$. We
immediately obtain the following result.

\begin{corollary}[Upper bound, tabular setting]\label{corr: megastate}
For any $\epsilon,\delta\in (0,1)$, UCB-VI applied on to the megastate-MDP $\Mcal^m$ learns an $\epsilon$-optimal policy for the original $m$-step decodable POMDP with probability greater than $1-\delta$ given $O\rbr{O^mA^mpoly(H) \log\rbr{1/\delta}/\epsilon^2}$ samples.
\end{corollary}

We remark that the sample complexity scales exponentially with the decoding length $m$. 
The next lower-bound verifies the necessity of the $O(A^m)$ term in the upper bound, so some exponential dependence is required. 
It follows by a reduction to the lower bound of~\citet{krish2016hardness}; we show that their construction is, in fact, an $m$-step decodable POMDP. 
This yields the following result.

\begin{proposition}[Lower bound, tabular setting] \label{prop: tabular_lower}
There exists an $m$-step decodable MDP that requires at least $\Omega(A^m/\epsilon^2)$ samples to find an $\epsilon$-optimal policy.
\end{proposition}

Thus the $A^m$ dependence in the megastate reduction is optimal,
although it is not clear whether the $O^m$ dependence is necessary, which
we discuss in more detail in \Cref{sec:conclusion}. Regardless,
the megastate reduction is a reasonable approach for $m$-step
decodable POMDPs when the observation space is small, but, in many
applications, the observations represent complex objects (like images
or high-dimensional data) so that even linear in $O$ dependence is
unsatisfactory.
Such problems lie outside the
scope of tabular methods, and a fundamentally different approach is
required.

%% file: sections/golf.tex

In this section we present our main results which address
the \emph{rich observation} setting, where the number of observation
$O$ is extremely large or infinite. The standard approach to tackle
such problems is via \emph{value function approximation}: we assume
access to a function class $\Fcal$ of candidate $Q$-value functions.
Given such a class, the goal is to learn a near-optimal policy with
sample complexity scaling with the statistical complexity of
$\Fcal$---in our case the log covering number
$\log \Ncal_{\Fcal}$---but independent of the size of the observation
space. In this section, we develop an algorithm for rich observation
$m$-step decodable POMDPs and analyze its sample complexity.



\begin{algorithm*}[t]
\caption{$m$-\golf : \golf\ for $m$-step decodable POMDP} \label{alg:golf}
 \begin{algorithmic}[1]
 \STATE \textbf{Initialize}:  $\Dcal_1,\dots,\Dcal_H\leftarrow \emptyset$, $\Bcal^0 \leftarrow \Fcal$.
 \STATE \textbf{Estimate} value of initial state by collecting $K_{\rm est}$ episodes and only keeping their first observations, denoted by $\hat{o}^1_1,\dots,\hat{o}^{K_{\rm est}}_1$. For $f \in \Fcal$, define
 \begin{equation*}
 	\hat{f}_1 = (1/K_{\rm est}) \sum_{i=1}^{K_{\rm est}} f(\hat{o}^i_1,\pi_f(\hat{o}^i_1))
 \end{equation*}
 \FOR{\textbf{epoch} $k$ from $1$ to $K$} 
 \STATE \textbf{Choose policy} $\pi^k = \pi_{f^k}$, where $f^k = \argmax_{f \in \Bcal^{k-1}} \hat{f}_1$. 
 \label{line:alg_golf_optimistic}
\FOR{\textbf{step} $h$ from $1$ to $H$}
\STATE \textbf{Collect}  $z_h=(o_{h-m+1},a_{h-m+1},\dots,o_h)$, $a_h$, $r_h$, and $o_{h+1}$ by executing $\pi^k$ at step $1,\ldots,h-m$ and taking  action uniformly at random at step $h-m+1,\dots,h$. 
\STATE \textbf{Augment} $\Dcal_h=\Dcal_h\cup (z_h,a_h,r_h,o_{h+1})$ for all $h\in[H]$.
\ENDFOR

\STATE \textbf{Update}
\begin{equation*}
	\Bcal^{k}=\left\{ f \in \Fcal:\  \Lcal_{\Dcal_h}(f_h,f_{h+1}) \leq \inf_{g \in \Gcal_{h}} \Lcal_{\Dcal_h}(g,f_{h+1}) + \beta \ \mbox{for all }h\in[H]\right\},
\end{equation*}
\vspace{-3mm}
\hspace{+10mm}
\begin{equation*}
\mbox{where }	\Lcal_{\Dcal_h}(\xi_{h},\zeta_{h+1}) = \sum_{(z_h,a_h,r_h,o_{h+1}) \in \Dcal_h}[\xi_h(z_h,a_h)-r_h -\max_{a' \in \Acal} \zeta_{h+1}(z_{h+1},a')]^2.
\end{equation*}
\ENDFOR
\STATE \textbf{Output} $\pi^{\rm out}$ uniform mixture policy over $\{\pi^k\}_{k=1}^{K}$.
 \end{algorithmic}
\end{algorithm*}

Our algorithm, which we call $m$-\golf, is displayed
in \Cref{alg:golf}. It is an adaptation of the \golf~algorithm,
developed by~\citet{jin2021bellman}, for the rich observation MDP
setting. $m$-\golf~itself differs from \golf~only in one seemingly
minor way, although this is quite critical for our analysis. Before
turning to this difference, let us review the high-level algorithmic
approach.

\golf, and $m$-\golf, are optimistic algorithms that maintain a 
confidence-set of plausible $Q$-value functions, and act
optimistically with respect to this set. Given a function class
$\Fcal$, we first collect a few observations $o_1$ and estimate the
predicted initial value, i.e., $\EE\sbr{f(o_1,\pi_f(o_1))}$, for each
$f \in \Fcal$.  Then, we initialize the confidence set
$\Bcal^0 \gets \Fcal$ and empty datasets $\{\Dcal_h\}_{h=1}^H$, one
for each time step. Then for each epoch $k \in [K]$ we follow three steps:
\begin{enumerate}
\item \emph{Optimistic planning}. Compute the function $f \in \Bcal^{k-1}$ with largest predicted initial value.
\item \emph{Data collection.} Collect one trajectory by following $\pi_{f_k}\circ_{m(h)}\unif(\Acal)$ for each $h \in [H]$. That is we collect $h$ trajectories total, rolling in with the greedy policy $\pi_{f_k}$ until time $h-m$ and rolling out randomly.
\item \emph{Refine the confidence set.} Update the confidence set to $\Bcal^{k}$ using the newly collected trajectories. The confidence set is designed so that $Q^\star \in \Bcal^k$ for all
$k \in [K]$ and that all functions in $\Bcal^k$ have low squared
Bellman error on the data collected in the previous
episodes.
\end{enumerate}

After iterating through these steps for several epochs,
$m$-\golf~outputs uniform mixture over all previous policies $\{\pi^k\}_{k=1}^{K}$.



The main difference between \golf~and $m$-\golf~is in the data
collection procedure. Instead of collecting $H$ trajectories per
epoch, \golf~collets a single trajectory where all actions are taken
by the greedy policy $\pi_{f_k}$. On the other hand, in $m$-\golf, we
interrupt the greedy policy and execute random actions so that the
tuple $z_h$ that is added to $\Dcal^h$ is collected from
$\pi_{f_k}\circ_{m(h)}\unif$. At face value, this modification is
relatively benign, but we will see how interrupting the greedy policy
is critical to establishing sample complexity guarantees in the
$m$-step decodable POMDP.

We analyze $m$-\golf~in two settings. The first is where the
underlying/latent MDP is tabular, meaning that $S$ and $A$ are
small. The second setting is where the latent MDP has a linear or low
rank structure. Our first theorem provides a sample complexity
guarantee for $m$-\golf~when the latent dynamics are tabular.


\begin{theorem} 
\label{thm:golf_regret}
Under Assumptions~\ref{asm:nstep_decode}, \ref{asm:realizability}, and \ref{asm:generalized_completeness}, there exists an absolute constant $c$ such that for any $\delta \in (0,1]$ and $\epsilon > 0$, if we choose
\begin{equation*}
\begin{aligned}
	&K_{\rm est} = c \cdot \Big( \log[\Ncal_{\Fcal}(\epsilon)/\delta]/\epsilon^2 \Big)\\
	&\beta = c \cdot \Big(\log\big[\Ncal_{\Gcal}(\rho) K H / \delta\big]+K\rho\Big)\\
	&{\rho =  \epsilon^2 \cdot \big[H^2A^m S\log[S/\epsilon]\big]^{-1}}
\end{aligned}
\end{equation*}
 in $m$-\golf\ (\cref{alg:golf}), then the output policy $\pi^{\rm out}$ is $\Ocal(\epsilon)$-optimal with probability at least $1-\delta$ if
\begin{equation*}
	{K \geq \tilde{\Omega}\left(\frac{H^2A^mS}{\epsilon^2} \cdot \log\left[\frac{\Ncal_{\Gcal}(\rho)}{\delta}\right] \right).}
\end{equation*}    
\end{theorem}

\Cref{thm:golf_regret} establishes a sample complexity bound for 
$m$-\golf~scaling as
$\textrm{poly}(S,A^m,H,\textrm{comp}(\Fcal,\Gcal),1/\epsilon)$ where
$\textrm{comp}(\cdot)$ is our measure of statistical complexity.
Unlike the megastate reduction, there is no explicit dependence on the
size of the observation space $O$; instead the bound scales with the
complexity of the function class, which allows us to exploit domain
knowledge and inductive biases when deploying the algorithm. 
In addition, the bound exhibits a linear dependence on $S$, the
cardinality of the latent state space. This dependence matches GOLF~\cite{jin2021bellman} and improves over previous upper
bounds for the Block MDP
case~\citep{jiang2017contextual,du2021bilinear}, which we recall is a
special case with $m=1$. 
We emphasize that these previous analyses do
not seem to yield guarantees when $m>2$, as we will see
in \Cref{sec:proof_overview}.


%% file: sections/linear.tex

In this subsection, we show that $m$-\golf~ extends to the setting where the number of state $S$ is also large.
Specifically, we consider the case where the latent MDP is a linear MDP~\cite{jin2020provably}---there exists an unknown feature map $\psi: \Scal\times\Acal\rightarrow \mathbb{R}^{d_{lin}}$ such that the transition dynamics are linear in $\psi$. Interestingly, we show that $m$-\golf\ is still applicable without change. It retains a similar sample complexity guarantee where we replace the dependence on $S$ with a dependence on the latent dimensionality $\dlin$.

Formally, a linear MDP is defined as follows:
\begin{definition}[Linear MDP]
\label{asm:linear}
An MDP $\Mcal = (\Scal, \Acal, H, \PP, r)$ is said to be a linear with a feature map $\psi :\Scal \times \Acal \rightarrow \mathbb{R}^{\dlin}$, if for any $h \in [H]$:
	 There exists $\dlin$ unknown (signed) measures $\bmu_h=\{\mu_h^{(1)},\dots,\mu_h^{(\dlin)}\}$ over $\Scal$ such that for any $(s,a) \in \Scal \times \Acal$ we have 
	 $$\PP_h(\cdot\mid s,a)=\langle\bmu_h(\cdot),\psi(s,a)\rangle$$
We assume the standard normalization:
$\norm{\psi(s,a)} \leq 1$ for all $(s,a) \in \Scal \times \Acal$,
$\norm{\int v(s) \bmu_h(s)}_2 \leq \sqrt{\dlin}$ for all
$h \in [H]$ and $v$ with $\|v\|_{\infty} \leq 1$.
\end{definition}



The following result gives a sample complexity guarantee for $m$-\golf~in the more general linear $m$-step decodable POMDP model.
\begin{theorem} 
\label{thm:golf_linear_regret}
Under Assumptions~\ref{asm:nstep_decode}, \ref{asm:realizability}, and \ref{asm:generalized_completeness} and assuming linear latent MDP; there exists an absolute constant $c$ such that for any $\delta \in (0,1]$ and $\epsilon > 0$, if we choose
\begin{equation*}
\begin{aligned}
	&K_{\rm est} = c \cdot \Big( \log[\Ncal_{\Fcal}(\epsilon)/\delta]/\epsilon^2 \Big)\\
	&\beta = c \cdot \Big(\log\big[\Ncal_{\Gcal}(\rho) K H / \delta\big]+K\rho\Big)\\
	&{\rho = \epsilon^2 \cdot \big[H^2A^m\dlin \log[\dlin/\epsilon]\big]^{-1}}
\end{aligned}
\end{equation*}
 in $m$-\golf\ (\cref{alg:golf}), then the output policy $\pi^{\rm out}$ is $\Ocal(\epsilon)$-optimal with probability at least $1-\delta$ if
\begin{equation*}
	{K \geq \tilde{\Omega}\left(\frac{H^2 A^m \dlin}{\epsilon^2} \cdot \log\left[\frac{\Ncal_{\Gcal}(\rho)}{\delta}\right] \right).}
\end{equation*}    
\end{theorem}

\Cref{thm:golf_linear_regret} is almost the same as
\Cref{thm:golf_regret} with the dependency on the number of
latent state $S$ replaced by the ambient dimensionality $\dlin$. As a
result, \Cref{thm:golf_linear_regret} can apply to the case where the
number of state $S$ is extremely large or even infinite, as long as
the underlying MDP has a linear structure.

%% file: sections/proof_overview.tex

In this section we elaborate on the main challenges in analysis, explain our main technique and provide a proof overview for~\pref{thm:golf_regret}.  For clarity, we will focus on the special case of $2$-step decodable POMDP in this setting. We refer reader to Appendix \ref{app:golf} for cases where $m > 2$.



\subsection{Challenges: Bellman Rank is Prohibitively Large}\label{sec: challenges}
We first note that existing postive results for RL algorithms with general function approximation such as \olive~\cite{jiang2017contextual}, \golf~\cite{jin2021bellman} all rely on the structural properties that certain complexity measure on the Bellman error is small. One such complexity is the Bellman rank~\cite{jiang2017contextual}, which explains the tractability of block MDP (the special case of $m$-step decodable POMDP with $m=1$).

Consider the Bellman error at the $h^{th}$ time step of a function $f\in \Fcal$ when executing roll-in policy $\pi$, given by
\begin{equation*}
	\Ecal_h(\pi,f) = \EE[(f_h-\Tcal_h f_{h+1})(z_h,\pi_f(z_h)) \mid a_{1:h-1} \sim \pi ].
\end{equation*}
Bellman rank is defined as the smallest integer $M$ such that the Bellman error can be factorized as inner product in $M$ dimensional linear space. That is, there exists $\zeta,\xi\in \mathbb{R}^M$ such that $\Ecal_h(\pi,f) = \inner{\zeta(\pi)}{\xi(f)}$. 

Intuitively, Bellman rank describes how much information is shared among past (roll-in policy $\pi$) and future (value function $f$) at step $h$. In the special case of $1$-step decodable POMDP, it suffices to consider $1$-step policy where the choice of action $a_h$ only depends on the current observation $o_h$. In this case, given the state $s_h$ at the current step $h$, the past---$(s, o, a)_{1:h-1}$ (which only depends on roll-in policy $\pi$) is completely independent of the future---$(o_h, a_h, (s, o, a)_{h+1:H})$ (which only depends on function $f$). Therefore, it can be shown the Bellman rank of $1$-step decodable POMDP (i.e. block MDP) is upper bounded by the number of states $S$ \cite{jiang2017contextual}.

However, such independent structure completely collapses in $2$-step decodable POMDP, where we must consider $2$-step policy. Due to the nature of such policies, the choice of action $a_h$ not only depends on the current observation $o_h$, but also the observation and action in the previous step $o_{h-1}, a_{h-1}$ (as shown in Figure \ref{fig: moment matching policy} blue box). Therefore, conditioning $s_h$, the past is no longer independent of the future. This can potentially lead to very large Bellman rank.

Formally, our next result shows that the Bellman rank in $2$-step decodable POMDP can be prohibitively large---there exists examples where the Bellman rank can be lower bounded by the cardinality of the observation space $\Omega(O)$. This is highly undesirable in the rich observation setting where $O$ can be even infinite. Furthermore, we also show that \olive\ algorithm---which was proposed in \cite{jiang2017contextual} to solve all RL problems with small Bellman rank---needs at least $\Omega(O)$ samples to find an $O(1)$ optimal policy. 

\begin{proposition}[Bellman rank of $m$-step decodable POMDP is large]\label{prop: bellman rank of an m step dec pomdp is large}
There exists a $2$-step decodable POMDP $\Mcal$ and a function class $\Fcal$ such that the Bellman rank of $\rbr{\Mcal,\Fcal}$ is $\Omega(O)$. Additionally, \olive\ instantiated with $\Fcal$ requires $\Omega(O)$ samples to find an $o(1)$ optimal policy.
\end{proposition}

This highlights the challenge on directly applying existing results or techniques to solve $m$-step decodable POMDPs. Although \olive\ solves a $1$-step decodable POMDP---namely, a block MDP---it fails in solving an $m$-step decodable POMDP for $m \ge 2$.



\subsection{Proof Overview \& Moment Matching Policy}
\label{sec:moment_matching}

\begin{figure}
\centering
\includegraphics[width=0.45\textwidth]{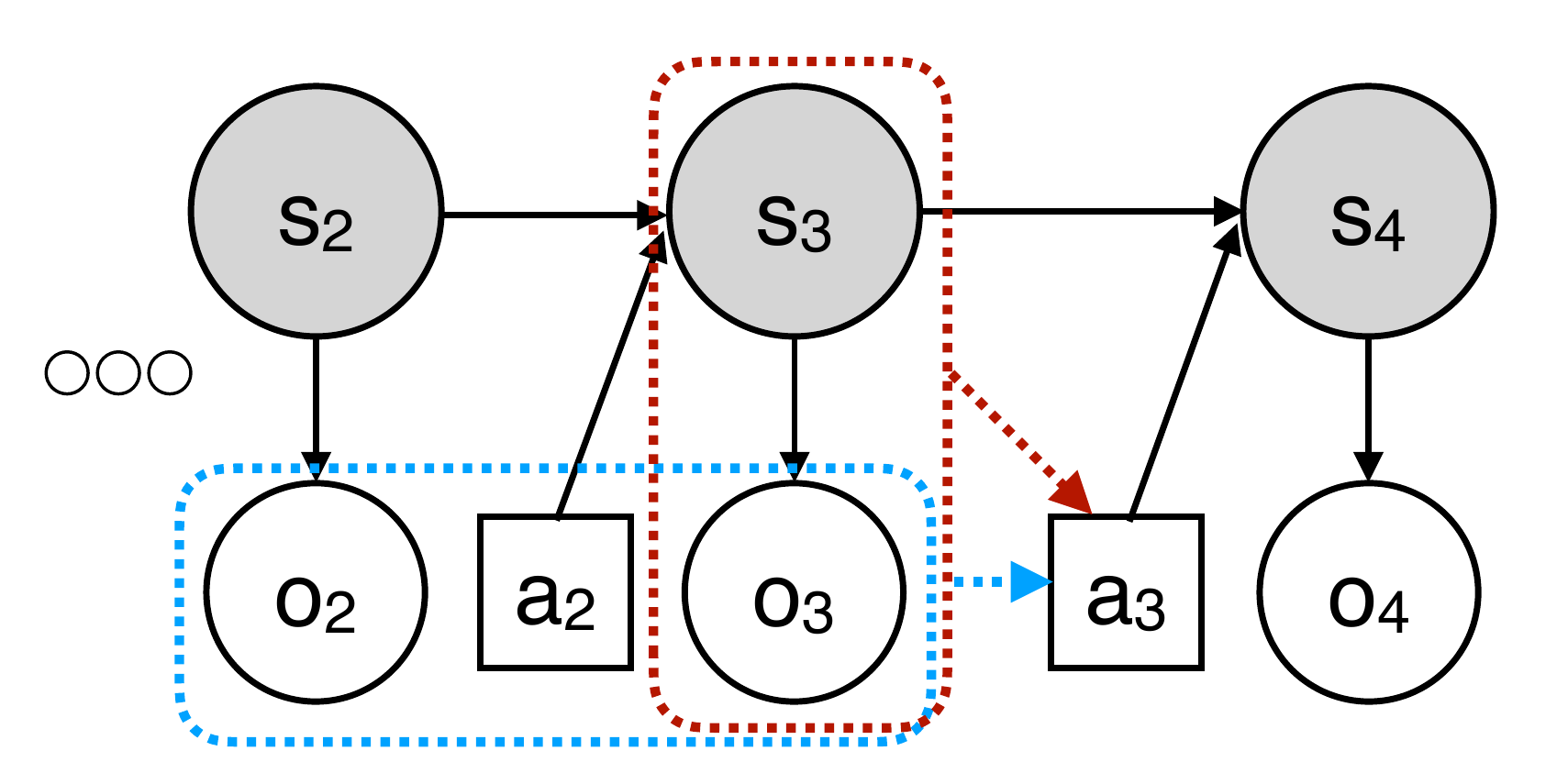}
\caption{An illustration of the dependency structure of a moment matching policy, depicted in red, and a regular policy, depicted in blue, in a $2$-step decodable POMDP. The moment matching policy $\mu^{\pi,h+1}$ selects action $a_h$ based on the state $s_h$ and observation $o_h$ to match the distribution $\PP^\pi[a_h \mid s_h,o_h]$. It breaks the dependence on the history by marginalizing out $(o_{h-1},a_{h-1})$, but correctly matches the distribution $\PP^\pi[o_{h+1},a_h,o_h]$. }
\label{fig: moment matching policy}
\end{figure}

Our main proof idea revolves around breaking the complicated dependencies introduced by multiple-step policies, which requires a number of crucial observations.

Our first key observation is that, in order to establish the sample complexity for \golf~algorithm, we don't necessarily need to prove the low rank structure of the Bellman error. We only need to alternatively identify an auxiliary function $\Ecal^\star_h(\pi,f)$ which satisfies the following two properties (see formal statement in~\pref{lem:golf_property}):
\begin{enumerate}
\item \emph{Matches with standard bellman error when $\pi = \pi_f$}: $$\Ecal^\star_h(\pi_f,f) =  \Ecal_h(\pi_f,f).$$
\item  \emph{Has a low-rank decomposition}:  $$\Ecal^\star_h(\pi,f)= \inner{\zeta(\pi)}{\xi(f)}.$$ 
for some $\zeta(\cdot),\xi(\cdot)\in \mathbb{R}^M$ with small $M$,
\end{enumerate}

This discovery gives us a lot extra freedom in designing the functional form of the $\Ecal^\star_h$. In particular, for $2$-step decodable POMDP, we define $\Ecal^\star_h$ to be the normal Bellman error but with the policy at step $h-1$ changed from roll-in policy $\pi$ to a new policy $\mu_f$ which depends only on $f$ instead of $\pi$.
\begin{align*}
 \Ecal^\star_h(\pi,f) \equiv 
 \EE[(f_h-\Tcal_h f_{h+1})(z_h,\pi_f(z_h)) \mid a_{1:h-1} \sim \pi \circ_{h-1} \mu_{f} ].
\end{align*}

The second key observation is that we can choose $\mu_f$ in a form which breaks the dependency and allows low-rank dependency. Concretely, instead of choosing $\mu_f$ to be standard $2$-step policy where $a_{h-1}$ will then depend on $(o_{h-2}, a_{h-2}, o_{h-1})$, we choose $\mu_f$ to be the policy that only depends on $(s_{h-1}, o_{h-1})$ (See Figure \ref{fig: moment matching policy} red box). 
The benefit of considering such policy is that now conditioned on $s_{h-1}$ at step $h-1$, the past---$(s, o, a)_{1:h-2}$ (which only depends on roll-in policy $\pi$) is now independent of the future---$(o_{h-1}, a_{h-1}, (s, o, a)_{h:H})$ (which only depends on function $f$). This immediately leads to a low-rank decomposition of $\Ecal^\star_h(\pi,f)$ with rank $M = S$.

Our third key observation is that we can carefully choose the value of $\mu_f$ within the form specified above, so that $\Ecal^\star_h(\pi,f)$ matches the Bellman error $\Ecal_h(\pi,f)$ when roll-in policy is the greedy policy of $f$, i.e. $\pi = \pi_f$. This is done by the idea of ``\emph{moment-matching}'', which is the reason we call policy $\mu_f$ the ``moment matching policy''. Specifically, we choose policy $\mu_f$ such that
$$\mu_f(a_{h-1}|(o, s)_{h-1}) = \EE_{\pi_f} [\pi_{f}(a_{h-1} | z_{h-1}) | (o, s)_{h-1}] $$
which is policy of $\pi_f$ averaging over all trajectories with $(o, s)_{h-1}$ fixed. The most important property of this policy is that the joint distributions over $z_h$ for policy $\pi_f$ and policy $\pi_f \circ_{h-1} \mu_f$ (which switches at time step $h-1$) are the same. In symbol:
$$ P_{\pi_f}(z_h) = P_{\pi_f \circ_{h-1} \mu_f}(z_h)$$
This directly leads to the matching in the Bellman error. This finishes our construction of $\Ecal^\star_h(\pi,f)$ satisfying the two properties mentioned earlier and the main part of proof overview.

Finally, we comment that our construction of $\mu_f$ depends on the latent state $s$ which can not be observed in POMDP.
Nevertheless, $m$-\golf\ bypasses this problem by executing a uniform action for $m$ time steps, instead of executing $\mu_f$; taking the uniform action for the last $m$ time steps allows us to upper bound $\Ecal^\star_h(\pi,f)$ using the importance sampling trick, while only suffering an $A^m$ degradation in the sample complexity. Such factor is necessary according to~\pref{prop: tabular_lower}.

%% file: sections/conclusion.tex

In this paper, we initiate the study of $m$-step decodable POMDPs as a
model for understanding the role of short-term memory in sequential
decision making. We consider both the tabular and function
approximation setting and obtain results that scale exponential with
the memory window rather than the horizon, which could be much
larger. In the function approximation case, our techniques rely
crucially on the moment matching policy to break dependency on the
history, and we hope this concept may be useful in other settings with
partial observability.

We believe our progress on understanding short-term memory is just
scratching the surface and there are many questions that remain open
even in the $m$-step decodable POMDP model. The most basic question
pertains to the tabular setting, where the upper bound in \cref{corr:
megastate} and the lower bound in \cref{prop: tabular_lower} differ by
an $O^m$ factor. Instantiating $m$-\golf~ in the tabular setting also
incurs an $O^m$ factor. On the other hand, the next result shows that
by using a carefully constructed policy class in an importance
sampling approach, we can avoid the $O^m$ factor in exchange for an
$A^H$ factor, which could be more favorable in some
settings. See \cref{app: H step decodable POMDP} for details and the
proof.

\begin{restatable}{proposition}{HstepDecPOMDPprop}\label{prop: H step dec POMDP}
There exists an algorithm such that for any $m\leq H$ and any $m$-step decodable POMDP, the algorithm returns an $\epsilon$-optimal policy 
with probability
greater than $1-\delta$ given
$poly(A^H,O,S,H,\log(1/\delta))/\epsilon^2$ samples.
\end{restatable} 
Based on this result, we conjecture that the $O^m$ factor can be avoided
and that $A^m \textrm{poly}(H,S,O,A)$ is the optimal sample complexity
for $m$-step decodable POMDPs. However, this question remains open.

The second question concerns whether we can avoid completeness, as
defined in \cref{asm:generalized_completeness}, in the rich observation
setting. Intuition from prior works suggests that if we could replace
the squared bellman error constraint with one on the average Bellman
errors, then an algorithm and analysis similar to \olive~would
successfully do this. However, when working with average Bellman
errors, introducing the moment matching policy requires explicitly
importance weighting with them, meaning that we must use these moment matching
policies in the algorithm and not just the analysis. Unfortunately
since we do not know the moment matching policies (or a small class
containing them), this approach seems to fail.

We believe that characterizing the optimal sample complexity (in the
tabular setting) or removing the completeness assumption (in the rich
observation setting) will require new techniques and be a mark of
significant progress toward expanding our understanding of decision
making with short-term memory. We look forward to studying these
questions in future work.






%% file: appendix/appendix_main.tex

\section{Proofs for \cref{sec:megastate}}
\label{app: megastate}
\input{appendix/app_megastate}

\section{Proof for \cref{sec:golf} and \ref{sec:proof_overview}}
\label{app:golf}
\input{appendix/app_golf}

\subsection{Proof for \cref{thm:golf_linear_regret}}
\label{app:liner}
\input{appendix/app_linear}

\section{On $H$-Step Decodable POMDPs}
\label{app: H step decodable POMDP}
\input{appendix/app_H_ste_dec_pomdp}

\section{Proof for \cref{prop: bellman rank of an m step dec pomdp is large}}
\label{app:brank}
\input{appendix/app_brank}


%% file: appendix/app_megastate.tex
In this section we provide formal proofs for the results stated in \cref{sec:megastate}. 

\begin{proof}[Proof of \cref{prop: megastate MDP}]
We need to verify that $\Mcal^m$ is an MDP. To do so, we check that
the state space induces a Markovian dynamics and that the expected
reward is also a function of the state. These two properties follow
from the $m$-step decodability assumption.

\begin{itemize}
\item \emph{Reward depends on the states.} This holds since the reward is assumed to depend only on the current observation $o_h$ and the current observation is included in the megastate. Formally, for any $s^{m,h} = \rbr{o_h,o_{h-1},a_{h-1},\ldots, o_{\min(h-m,1)},a_{\min(h-m,1)}}\in \Scal^{m,h}$, any history $\Hcal$, and policy $\pi$,  it holds that
\begin{align*}
\EE_\pi\sbr{r | s^{m,h},\Hcal} = \EE_\pi\sbr{r | s^{m,h}} = r(o_h) 
\end{align*}
where $o_h\in s^{m,h}$ due to the assumption on the reward generation process of $m$-step decodable POMDP.

\item \emph{Transition model is Markov.} For any $s^{m,h} = \rbr{o^h_{h},o^h_{h-1},a^h_{h-1},\ldots, o^h_{\min(h-m,1)},a^h_{\min(h-m,1)}}\in \Scal^{m,h}$ and $s^n_{h+1} = \rbr{o^{h+1}_{h+1},o^{h+1}_{h},a^{h+1}_{h},\ldots, o^{h+1}_{\min(h+1-m,1)},a^{h+1+1}_{\min(h+1-m,1)}}\in \Scal_{h+1}^m$, any action $a_h^h\in \Acal$, any history $\Hcal$ and any policy $\pi$  it holds that
\begin{align*}
&\PP_\pi\rbr{s^{m,h+1} \mid s^{m,h},a_h, \Hcal} = \PP_\pi\rbr{o_{h+1} \mid s^{m,h},a_h, \Hcal} \cdot \prod_{j=\min(h+1-m,1)}^{h} \delta\rbr{o_j^{h+1} = o_j^h, a_j^{h+1}=a_j^h}
\end{align*}
Finally, observe that by the $m$-step decodability assumption it holds that
\begin{align*}
&\PP_\pi\rbr{o_{h+1} \mid \phi^\star(s^{m,h})= s ,a_h, \Hcal}  = \mathbb{O}_{h+1}\rbr{o_{h+1} \mid s_{h+1}}\PP\rbr{s_{h+1} \mid \phi^\star(s^{m,h})= s ,a_h},
\end{align*}
where the last relation holds by the Markov assumption of the latent model. This shows that
\begin{align*}
\PP_\pi\rbr{s^{m,h+1} \mid s^{m,h},a_h, \Hcal} = \PP\rbr{s^{m,h+1} \mid s^{m,h},a_h},
\end{align*}
and hence the dynamics are Markovian.
\end{itemize}

{Lastly, we elaborate on the optimality of any optimal policy of $\Mcal^m$; that is, any optimal policy of $\Mcal^m$ is an optimal policy of the $m$-step decodable POMDP. First, observe that the optimal policy of the latent MDP that underlies the $m$-step decodable POMDP is also the optimal policy of the $m$-step decodable POMDP. 

Further, since the latent state is decodable from a suffix of length $m$ of the history, any state in $\Scal^{m,h}$ (that represents a reachable suffix) can decode the latent state. Hence, the optimal policy on the latent MDP can be executed based on the states in $\Scal^{m,h}$. Thus, an optimal policy of $\Mcal^{m}$ is also an optimal policy of the $m$-step decodable POMDP; otherwise, an optimal policy of the latent MDP is not optimal for the $m$-step decodable POMDP. }

\qedhere
\end{proof}

\begin{proof}[Proof of \cref{corr: megastate}]
The sample complexity follows immediately from a standard
online-to-batch conversion of the minimax optimal regret bound
in~\cite{azar2017minimax}, combined with \cref{prop: megastate
  MDP}. In particular, the online-to-batch conversion gives
$\tilde{O}(HSA\log^2(1/\delta)/\epsilon^2)$ sample complexity in an
MDP with $S$ states and $A$ actions. By \cref{prop: megastate MDP} we
have an MDP with $O^mA^{m-1}$ states, so the result follows.
\end{proof}

\begin{proof}[Proof sketch of \cref{prop: tabular_lower}]
We construct a simple $m$-step decodable POMDP with horizon $m$, two
states per layer and two actions. The construction and argument are
identical to the one in~\cite{krish2016hardness}, so we only sketch
the construction here. It is a standard ``combination lock''
construction, with $A$ actions and no observations, but where the
state is decodable from the past actions.

In particular, the agent starts in the ``good state'' $g_1$ and at
each time step $h$ can be either in the good state $g_h$ or the ``bad
state'' $b_h$. From the good state, a special action $a_h^\star$
transits to the next good state, while all other actions (from both
good or bad state) transit to the next bad state $b_{h+1}$. At the
last time step the agent gets reward for being in state $g_m$. There
are no observations (or there is a trivial observation), but note that
the latent state is decodable using the history of actions. Thus
provided the horizon $H\leq m$ the process is $m$-step decodable.

Intuitively, the construction requires the agent to try all $A^m$
action sequences before finding the reward. More formally this
construction embeds an $\Omega(A^m)$ armed bandit problem resulting in
a sample complexity lower bound of $\Omega(A^m/\epsilon^2)$. We refer
the reader to~\cite{krish2016hardness} for more details.

\end{proof}

%% file: appendix/app_golf.tex
In this section we provide formal proofs for the results stated in \cref{sec:golf} and \ref{sec:proof_overview}. 

\subsection{Properties of Moment Matching Policy}
\label{app:mmpolicy}
\input{appendix/app_moment_matching}

\subsection{Concentration lemmas}
We start with the following lemma, which is quite similar to Lemmas 39
	and 40 in \citealt{jin2021bellman}. The lemma shows that: (1) with
	high probability any function in the confidence set at the $k^{\textrm{th}}$ iteration has low
	Bellman error over the data distributions from visited in the previous iterations at all layers $h \in [H]$ and (2)
    the optimal value
	function is inside the confidence set with high probability.

\begin{lemma}
\label{lem:golf_concentration}
For any $\rho > 0$ and $\delta \in (0,1)$, if we run \cref{alg:golf} with $\beta = c \Big(\log\big[KH\Ncal_{\Gcal}(\rho)/\delta\big] + K\rho \Big)$ where $c >0$ is an absolute constant, then with probability at least $1-\delta$, we have
\begin{enumerate}
	\item $\sum_{i=1}^{k-1} \EE \Big[  \big(f^k_h(z_h,a_h) - (\Tcal_h f^k_h)(z_h,a_h)\big)^2  \mid a_{1:h-m}\sim \pi^i, a_{h-m+1:h} \sim \mathrm{unif}(\Acal) \Big] \leq \Ocal(\beta)$ for all $(k,h) \in [K] \times [H]$,
	\item $Q^\star \in \Bcal^k$ for all $k \in [K]$.
\end{enumerate}
\end{lemma}
\begin{proof}[Proof of \cref{lem:golf_concentration}]
	The proof relies on a standard martingale concentration inequality
	(e.g., Freedman’s inequality), the construction of our confidence
	set, and our generalized completeness assumption
	(\cref{asm:generalized_completeness}). The argument is almost
	identical to the proofs of Lemma 39 and 40
	in \citealt{jin2021bellman} and therefore omitted for
	brevity.  \end{proof}

\begin{lemma} 
\label{lem:golf_initial_estimation}
For any $\delta \in (0,1)$, if we choose $K_{\rm est} = c \cdot \big(\log[\Ncal_{\Fcal}(\rho_{\rm est})/\delta]/\rho_{\rm est}^2\big)$ where $c > 0$ is some absolute constant; then, with probability at least $1-\delta$ for any $f \in \Fcal $, we have
	\begin{equation*}
		| \hat{f}_1 - \EE_{s_1} \big[f_1(o_1,\pi_{f}(o_1)) \big] | \leq \mathcal \Ocal(\rho_{\rm est} ).
	\end{equation*}
\end{lemma}
\begin{proof}
	The proof follows from applying uniform concentration argument over a $\rho_{\rm est}$-cover of $\Fcal $; then, a covering argument finishes the proof.  
\end{proof}

\subsection{Eluder Dimension}
In this section we describe complexity measure \emph{Eluder dimension} proposed by \citet{russo2013eluder} since it has been used in the analysis of the original \golf\ algorithm \cite{jin2021bellman}. 
\begin{definition}[$\epsilon$-Independence] Let $\Wcal$ be a function class defined over domain $\Ycal$ and $y^1,\dots,y^n,\bar{y}$ be elements in $\Ycal$. We say $\bar{y}$ is $\epsilon$-independent with respect to  $\Wcal$, if there exists $w \in \Wcal$ such that $\sqrt{\sum_{i=1}^n [w(y^i)]^2} \leq \epsilon$, but $|w(\bar{y})| > \epsilon$.  
\end{definition}
\begin{definition}[Eluder Dimension] The Eluder dimension $\eludim(\Wcal,\epsilon)$, is the length of the longest sequence of $\{y^1,\dots,y^n\}$ in $\Ycal$, such that there exists $\epsilon' \geq \epsilon$ where $y^i$ is $\epsilon'$-independent of $\{y^i,\dots,y^{i-1}\}$ with respect to $\Wcal$ for all $i \in [n]$.
	
\end{definition}

The following proposition shows that if $\Wcal$ has a low rank structure with rank $d$, then the Eluder dimension can be upper bounded by $\tilde{\Ocal}(d)$.

\begin{proposition}[Proposition 6 in \citealt{russo2013eluder}]  
\label{prop:bellman_to_eluder}
Suppose for any $w \in \Wcal$ and any $y \in \Ycal$, we have
$
	w(y) = \langle \zeta (y) , \xi(w) \rangle,
$
	where $\zeta(y) , \xi(w) \in \mathbb{R}^{d}$ satisfying $\norm{\zeta(y)} \cdot \norm{\xi(w)} \leq \gamma$. Then we have,
	\begin{equation*}
		\eludim(\Wcal,\epsilon) \leq \Ocal\big(1+d\log[1+\gamma/\epsilon^2]\big).
	\end{equation*}
\end{proposition}

The following lemma could be seen as an analogue to the standard elliptical potential argument for Eluder dimension that was proposed by \citet{russo2013eluder} and been used in analysis of \golf. The following lemma could be obtained from Lemma 41 in \citet{jin2021bellman} by setting the family of probability measures used in that lemma to be $\{\delta_{y} \mid y \in \Ycal \}$, where $\delta_y$ is the dirac measure centered at $y$.

\begin{lemma}[Simplification of Lemma 41 in \citealt{jin2021bellman}]
\label{lem:eluder_main}
	Given a function class $\Wcal$ defined over $\Ycal$ with $w(y)\leq C$ for all $(w,y) \in \Wcal \times \Ycal$; Suppose $\{y^i\}_{i=1}^K \subseteq \Ycal$ and $\{w^i\}_{i=1}^K \subseteq \Wcal$ satisfy that for all $k \in K$, $\sum_{i=1}^{k-1} [w^k(y^i)]^2 \leq \alpha$. Then for all $k \in [K]$ and $\omega > 0$, we have
	\begin{equation*}
		\sum_{i=1}^k | w^i(y^i) | \leq \Ocal\Big(\sqrt{\eludim(\Wcal,\omega)\alpha k}+\min\{k,\eludim(\Wcal,\omega)\}\cdot C + k\omega).
	\end{equation*} 
\end{lemma}

\subsection{Proof of \cref{thm:golf_regret}}
We use $\Ecal_h(\pi,f)$ to denote the Bellman error of function $f \in \Fcal$ at step $h$ using roll-in policy $\pi$, which is defined as
\begin{equation*}
	\Ecal_h(\pi,f) = \EE[(f_h-\Tcal_h f_{h+1})(z_h,\pi_f(z_h)) \mid a_{1:h-1} \sim \pi ].
\end{equation*}
In addition, we use $\Ecal^{\star}_h(\pi,f)$ to denote the Bellman error of function $f$ at step $h$ using roll-in policy $\pi$ for the first $h-m$ steps and $\nu^{\pi_f,h}$ (the moment matching policy for $\pi_f$) for $a_{m(h):h-1}$; namely,  
\begin{equation*}
	\Ecal^{\star}_h(\pi,f) = \EE[(f_h-\Tcal_h f_{h+1})(z_h,\pi_f(z_h)) \mid a_{1:h-m} \sim \pi , a_{m(h):h-1} \sim \nu^{\pi_f,h}].
\end{equation*}
The next lemma shows that $\Ecal^{\star}_h$ satisfies two important
properties that are critical to the rest of the proof. The first
property is that when $\pi = \pi_f$, $\Ecal_h$ and $\Ecal^{\star}$
coincide. The second property shows that $\Ecal^{\star}_h$ has low
rank or bilinear structure.

\begin{lemma}
\label{lem:golf_property}
For any policy $\pi$, any function $f \in \Fcal$, and any $h \in [H]$, we have
\begin{enumerate}
	\item $\Ecal_h(\pi_f,f)=\Ecal^{\star}_h(\pi_f,f)$
	\item $\Ecal^{\star}_h(\pi,f) = \langle \zeta_h(\pi), \xi_h(f) \rangle$ where $\zeta_h(\pi) , \xi_h(f) \in \mathbb{R}^{S}$ satisfy $\norm{\zeta_h(\pi)} \leq 1$ and $\norm{\xi_h(f)} \leq 2\sqrt{S}$.
\end{enumerate}
\end{lemma}

\begin{proof}[Proof of \cref{lem:golf_property}]
For item (1) define $\tilde{\pi}^h_f$ to be the policy that takes actions $a_{1:h-m} \sim \pi_f$ and $a_{m(h):h-1} \sim \nu^{\pi_f,h}$, and let $g : \Zcal_h \rightarrow [0,2]$ be defined as $g(z_h) = (f_h-\Tcal_h f_{h+1})(z_h,\pi_f(z_h))$. Then by item (1) of~\cref{lem:main_momentmatching} we have
\begin{align*}
			\Ecal^{\star}_h(\pi,f) &= \EE[(f_h-\Tcal_h f_{h+1})(z_h,\pi_f(z_h))) \mid a_{1:h-m} \sim \pi , a_{m(h):h} \sim \nu^{\pi_f,h}]\\
			&= \sum_{z_h \in \Zcal_h} P_{\tilde{\pi}^h_f}(z_h)\cdot g(z_h)
			= \sum_{z_h \in \Zcal_h} P_{\pi_f}(z_h)\cdot g(z_h)\\
			&= \EE[(f_h-\Tcal_h f_{h+1})(z_h,\pi_f(z_h))) \mid a_{1:h-1} \sim \pi]
			= \Ecal_h(\pi_f,f),
\end{align*}

Item (2) immediately follows from item (2) of \cref{lem:main_momentmatching} by selecting $g$ as $g(z_h) = (f_h-\Tcal_h f_{h+1})(z_h,\pi_f(z_h))$ and $\bar{\pi}= \pi_g$.
\end{proof}

The following corollary shows that Eluder dimension with respect to $\Ecal^\star$ is upper bounded by $\tilde{\Ocal}(S)$. The proof immediately follows from \cref{lem:golf_property} and \cref{prop:bellman_to_eluder}. 
\begin{corollary}
\label{cor:golf_eludim_S}
	Let $\Pi$ to be set of all $m$-step policies, and define $\Wcal^\star_{\Fcal} = \{ \Ecal^\star(\cdot,f):\Pi \rightarrow [0,2] \mid f \in \Fcal\}$, then
	\begin{equation*}
		\eludim(\Wcal^\star_\Fcal,e) \leq \Ocal\big(S\log[S/\epsilon]\big).
	\end{equation*}
\end{corollary}

Now we are ready to prove \cref{thm:golf_linear_regret}.

\begin{proof}[Proof of \cref{thm:golf_regret}]
	With probability at least $1-2\delta$ the events in \cref{lem:golf_concentration} and \cref{lem:golf_initial_estimation} holds. Under this good event, we proceed in several steps.

%

\paragraph{Step 1. Bounding the optimality gap by the Bellman error.}
\cref{lem:golf_concentration} guarantees that $\forall k \in [K]: \ Q^\star \in \Bcal^k$, this together with optimistic choice of $f^k$ (Line~\ref{line:alg_golf_optimistic} in \cref{alg:golf}), for all $k \in [K]$, we have:
\begin{equation*}
	 V^\star \leq \hat{Q}^\star_1 + \Ocal(\rho_{\rm est}) \leq \hat{f}^k_1 +  \Ocal(\rho_{\rm est}) \leq \EE_{s_1} \big[f^k_1(o_1,\pi_{f^k}(o_1)) \big] +  2\cdot\Ocal(\rho_{\rm est}).
\end{equation*}
It implies that $\sum_{k=1}^K \big(V^\star - V^{\pi^k}\big)
        {\le} \sum_{k=1}^K \EE_{s_1} \big[f^k_1(o_1,\pi_{f^k}(o_1)) \big] - V^{\pi^k} +  \Ocal(K\rho_{\rm est})$. We also have
\begin{align*}
        \EE_{s_1} \big[f^k_1(o_1,\pi_{f^k}(o_1)) \big] - V^{\pi^k}  
        \overset{(i)}{=} \sum_{k=1}^K \sum_{h=1}^H  \Ecal_h(\pi^k, f^k)
        \overset{(ii)}=\sum_{h=1}^H\sum_{k=1}^K \Ecal^{\star}_h(\pi^k,f^k),
\end{align*}
where $(i)$ is by standard policy loss decomposition (e.g., Lemma 1 in \citealt{jiang2017contextual}) and $(ii)$ is due to part (1) of \cref{lem:golf_property} since we have $\pi^k = \pi_{f^k}$. Therefore, we showed
\begin{equation*}
	\sum_{k=1}^K \big(V^\star - V^{\pi^k}\big)
        {\le} \sum_{h=1}^H\sum_{k=1}^K \Ecal^{\star}_h(\pi^k,f^k) + \Ocal(K\rho_{\rm est})
\end{equation*}

\paragraph{Step 2: Utilizing the confidence set.}
By \cref{lem:golf_concentration}, we have 
\begin{align*}
\sum_{i=1}^{k-1} \EE \Big[   \big((f^k_h-\Tcal_h f^k_{h+1})(z_h,a_h)\big)^2  \mid a_{1:h-m}\sim \pi^i, a_{h-m+1:h} \sim \mathrm{unif}(\Acal) \Big] \leq \Ocal(\beta) \quad \forall (k,h) \in [K] \times [H].
\end{align*}
It implies that 
\begin{equation*}
\begin{aligned}
	\sum_{i=1}^{k-1} [\Ecal_h^\star(\pi^i,f^k)]^2
	&\leq \sum_{i=1}^{k-1}\EE\Big[\big((f^k_h-\Tcal_h f^k_{h+1})(z_h,\pi_f(z_h))\big)^2 \mid a_{1:h-m} \sim \pi^i , a_{h-m+1:h} \sim \nu^{\pi_{f^k},h}\Big] \\
        &\leq A^m\sum_{i=1}^{k-1}\EE \Big[  \big((f^k_h-\Tcal_h f^k_{h+1})(z_h,a_h)\big)^2  \mid a_{1:h-m}\sim \pi^i, a_{h-m+1:h} \sim \mathrm{unif}(\Acal) \Big] \\
        &\leq \mathcal{O}(A^m\beta).
\end{aligned}
\end{equation*}
Here the $A^m$ factor arises to change measure from
$\nu^{\pi_{f^k},h}$ to the uniform distribution over actions
$a_{h-m+1:h}$.
\paragraph{Step 3: Utilizing Low-rank Structure.}
From previous step, we know that $\sum_{i=1}^{k-1} [\Ecal_h^\star(\pi^i,f^k)]^2 \leq A^m \beta$, Therefore if we invoke \cref{lem:eluder_main} and \cref{cor:golf_eludim_S} with
\begin{equation*}
\begin{cases}
	\Ycal = \Pi, \quad &\Wcal = \Wcal^\star_\Fcal = \{ \Ecal^\star(\cdot,f):\Pi \rightarrow [0,2] \mid f \in \Fcal\},\\
	\omega = \epsilon/H, \quad & \alpha = \Ocal(A^m \beta), \quad C=2,\\
	
\end{cases}
\end{equation*}
we obtain 
\begin{equation*}
	\frac{1}{K}\sum_{k=1}^K \Ecal_h^\star(\pi^k,f^k) \leq \Ocal \big( \sqrt{\frac{A^mS \log[S/\epsilon]\beta}{K}}+ \epsilon/H \big)
\end{equation*}

\paragraph{Step 4: Putting everything together}
Choosing $\rho_{\rm est} = \Ocal(\epsilon)$ and combining the conclusion of step 1 and step 3, we have
\begin{align*}
         &\frac{1}{K}\sum_{k=1}^K \big(V^\star - V^{\pi^k}\big) \leq \frac{1}{K}\sum_{h=1}^H\sum_{k=1}^K \Ecal^{\star}_h(\pi^k,f^k) \leq \Ocal \big( \sqrt{\frac{H^2 A^mS \log[S/\epsilon]\beta}{K}}+ \epsilon \big) + \Ocal(\epsilon). 
\end{align*}
By definition of $\pi^{\rm out}$, we have
\begin{equation*}
	\begin{aligned}
		V^{\star}-V^{\pi^{\rm out}} &= \frac{1}{K} \sum_{k=1}^K \big(V^\star - V^{\pi^k}\big) \leq \Ocal \big( \sqrt{\frac{H^2A^mS\log[S/\epsilon]\beta}{K}} \big) + \Ocal(\epsilon)\\
		&\overset{(i)}{\leq} \Ocal \Big( \sqrt{\frac{H^2A^mS \log[S/\epsilon]\log[KH\Ncal_{\Gcal}(\rho)/\delta]}{K}+H^2A^mS\log[S/\epsilon]\rho} \Big) + \Ocal(\epsilon)\\
		&\overset{(ii)}{\leq} \Ocal \Big( \sqrt{\frac{H^2A^mS \log[S/\epsilon]\log[KH\Ncal_{\Gcal}(\rho)/\delta]}{K}} \Big) + \Ocal(\epsilon)\\,
	\end{aligned}
\end{equation*}
where $(i)$ is follows from $\beta = c \Big(\log\big[KH\Ncal_{\Gcal}(\rho)/\delta\big] + K\rho \Big)$ as in \cref{lem:golf_concentration} and $(ii)$ is by picking $$\rho = \frac{\epsilon^2}{(H^2A^mS\log[S/\epsilon])}.$$
We need to pick $K$ such that 
\begin{equation*}
	\sqrt{\frac{H^2A^mS \log[S/\epsilon]\log[KH\Ncal_{\Gcal}(\rho)/\delta]}{K}} \leq \Ocal(\epsilon).
\end{equation*}
By simple calculations, one can verify that it suffices to pick
\begin{equation*}
	K \geq \Omega(\frac{H^2SA^m}{\epsilon^2}\cdot \log[HSA^m\Ncal_{\Gcal}(\rho)/(\delta\epsilon)]\cdot \log[S/\epsilon]),
\end{equation*}
which completes the proof.
\end{proof}

%% file: appendix/app_moment_matching.tex
We start with formal definition of \emph{moment matching policy}. For a policy $\pi$, we construct $\nu^{\pi,h}_{h'}$ for $h'\geq h-m$ such that it matches the distribution of the action $a_{h'}$ conditioning on latent states and observations from time step $h-m+1$ to time step $h$ under the sampling process of $\pi$.  For this reason we refer to $\nu^{\pi}$ as the \emph{moment matching policy} for $\pi$ (see~\pref{fig: moment matching policy} for illustration). Formally, we define it as follows:

\begin{definition}[Moment-Matching Policy for $\pi$] 
\label{def:moment_matching} Denote $m(h) = h-m+1$; Fix $h \in [H]$ and for $h' \in [m(h),h]$ we define  
$$
	x_{h'}=\big(s_{m(h):h'},o_{m(h):h'},a_{m(h):h'-1}\big) \in \Xcal_l,
$$
where $\Xcal_l = \Scal^l \times \Ocal^l \times \Acal^{l-1}$ and $l = h'-m(h)$. For a $m$-step policy $\pi$ and $h \in [H]$, we define the moment matching policy $\mu^{\pi,h}=\{\mu^{\pi,h}_{h'}: \Xcal_l \rightarrow \Delta (\Acal) \}_{h'=m(h)}^h$ as following:
\begin{equation*}
\begin{aligned}
	\mu^{\pi,h}_{h'}(a_{h'} \mid x_{h'})  
	\defeq \EE_\pi[ \pi_{h'}(a_{h'} \mid z_{h'}) \mid x_{h'}].
\end{aligned}
\end{equation*}	
By \cref{asm:nstep_decode}, states and therefore $x_h'$ is decodable by the history of actions and observations, therefore we let
\begin{equation*}
	\begin{aligned}
		\nu^{\pi,h}_{h'}(a_{h'} \mid o_{1:h'}, a_{1:h'-1}) 
		\defeq \mu^{\pi,h}_{h'} (a_{h'} \mid x_h').
	\end{aligned}
\end{equation*}
\end{definition}

As we discussed in \cref{sec:proof_overview}, we prove the following lemma that establishes two important properties of the moment matching policy.

\begin{lemma}
	\label{lem:main_momentmatching}
	For a fixed $h \in [H]$ and fixed $m$-step policies $\pi,\bar{\pi}$, define policy $\tilde{\pi}^h$ which takes first $m(h)-1$ actions from $\pi$ and remaining actions from $\nu^{\bar{\pi},h}$, i.e. $\tilde{\pi}^h= \pi \circ_{m(h)} \nu^{\pi,h}$. Then we have, 
	\begin{enumerate}
		\item If $\pi = \bar{\pi}$, for any $z_h \in \Zcal_h$, $P_{\pi}(z_h)=P_{\tilde{\pi}^h}(z_h)$
		\item For any function $g: \Zcal_h \rightarrow [0,1]$,
		\begin{equation*}
			\EE_{\tilde{\pi}^h}[g(z_h)]= \langle \zeta_h(\pi) , \xi_h(g,\bar{\pi})\rangle,
		\end{equation*}
		where $\zeta_h(\pi),\xi_h(g,\bar{\pi}) \in \mathbb{R}^S$ satisfying $\norm{\zeta_h(\pi)} \leq 1$ and $\norm{\xi_h(g,\bar{\pi})} \leq \sqrt{S}$.
	\end{enumerate}
\end{lemma}

Recall that we use the notation $m(h) = \min\{h-m+1,1\}$ and that we define $x_{h'} = (s_{m(h):h'}, o_{m(h):h'}, a_{m(h):h'-1})$ for $h' \in [m(h),h]$. By
definition of $\mu^{\pi,h}$ (as in \cref{def:moment_matching}), for
$h' \in [m(h),h]$ we have
\begin{align}
	\mu^{\pi,h}_{h'}(a_{h'} \mid x_{h'})P_\pi\big[ x_{h'}\big]
	= \sum_{(o,a)_{m(h'):m(h)-1}} \pi(a_{h'} \mid z_{h'}) P_\pi \big[ (o,a)_{m(h'):m(h)-1},x_{h'} \big]\label{eq:app:moment_matching_1}
\end{align}
We will this identity below.
\begin{proof}[Proof of \cref{lem:main_momentmatching}]
Recall that we define $\tilde{\pi}^h$ to take actions $a_{1:m(h)-1}$
according to $\pi$ and take actions $a_{m(h):h-1}$ according to the
moment matching policy $\nu^{\pi,h}$. 

\paragraph{Item 1.}
	We prove the first item by induction on $h' \in \{m(h),\ldots,h\}$, where the induction hypothesis is 	
	\begin{align*}
	\forall x_{h'} : \quad 
	P_\pi\big[ x_{h'} \big] = P_{\tilde{\pi}^h}\big[ x_{h'}\big]
	\end{align*}
	\begin{itemize}
		\item \textbf{Base case:} The base case is when $h' = m(h)$. In this case, $P_\pi\big[ (s,o)_{m(h)} \big] = P_{\tilde{\pi}^h}\big[ (s,o)_{m(h)} \big]$ since all actions up to $a_{m(h)-1}$ are taken by the same policy.
		\item \textbf{Induction step}: Let $h' \in \{m(h),\ldots,h\}$ and assume $P_\pi\big[ x_{h'-1} \big] = P_{\tilde{\pi}^h}\big[ x_{h'-1} \big]$. We have 
		\begin{align*}
			P_\pi(x_{h'+1}) &= P_\pi \big[(s,o,a)_{m(h):h'} , (s,o)_{h'+1} \big]\\
			&= \sum_{(o,a)_{m(h'):m(h)-1}} P_\pi \big[ (o,a)_{m(h'):m(h)-1}, x_{h'} , a_{h'}, (s,o)_{h'+1} \big]\\
			&= \sum_{(o,a)_{m(h'):m(h)-1}} \OO(o_{h'+1} \mid s_{h'+1})\PP(s_{h'+1} \mid s_{h'},a_{h'}) \pi(a_{h'} \mid z_{h'})P_\pi \big[ (o,a)_{m(h'):m(h)-1},x_{h'} \big]
		\end{align*}
		Similarly we have,
		\begin{align*}
			P_{\tilde{\pi}^h}(x_{h'+1}) &=  P_{\tilde{\pi}^h} \big[(s,o,a)_{m(h):h'} , (s,o)_{h'+1} \big]\\
			&=  P_{\tilde{\pi}^h} \big[ x_{h'} , a_{h'}, (s,o)_{h'+1} \big]\\
			&=  \OO(o_{h'+1} \mid s_{h'+1})\PP(s_{h'+1} \mid s_{h'},a_{h'}) \mu^{\pi,h}_{h'}(a_{h'} \mid x_{h'})P_{\tilde{\pi}^h}\big[ x_{h'}\big]\\
			&\overset{(i)}{=}\OO(o_{h'+1} \mid s_{h'+1})\PP(s_{h'+1} \mid s_{h'},a_{h'}) \mu^{\pi,h}_{h'}(a_{h'} \mid x_{h'})P_{\pi}\big[ x_{h'}\big],
		\end{align*}
		where $(i)$ uses the induction hypothesis. \cref{eq:app:moment_matching_1} implies that right-hand side of the two above expressions are equal, which completes the proof of induction step.
	\end{itemize}
	Now item 1 is immediate since the variables in $z_h$ are contained
    within $x_h$, in particular
	\begin{equation*}
		P_{\pi}(z_h)=\sum_{s_{m(h):h}}P_{\pi}(x_h)=\sum_{s_{m(h):h}}P_{\tilde{\pi}^h}(x_h)=P_{\tilde{\pi}^h}(z_h).
	\end{equation*}
	
	\paragraph{Item 2.} 
    Recall that here $\tilde{\pi}^h$ is defined to take actions
    $a_{1:m(h)-1} \sim \pi$ and $a_{m(h):h-1}\sim \nu^{\bar{\pi},h}$ where
    $\pi$ and $\bar{\pi}$ may not be equal. Since $\mu^{\bar{\pi},h}$ is defined
    to be independent of the past give $s_{m(h)}$ we have the
    factorization
	\begin{align*}
	  \EE_{\tilde{\pi}^{h}}[g(z_h)] 
      &= \sum_{s_{m(h)} \in \Scal} P_\pi(s_{m(h)}) \cdot \EE_{a_{m(h):h-1} \sim \mu^{\bar{\pi},h}}[g(z_h) \mid s_{m(h)}].
	\end{align*}
	We note that $\mu^{\bar{\pi},h}$ only depends on $(s,o)_{m(h):h-1}$ and $a_{m(h):h-2}$, thus the second term is independent of $\pi$ and only depends $g$ and $\bar{\pi}$. Defining 
	\begin{equation*}
	\begin{aligned}
			&\zeta_h(\pi) \defeq \big( P_\pi(s_{m(h)})\big)_{s_{m(h)} \in \Scal} \in \mathbb{R}^S \quad \textrm{and} \quad \xi_h(g,\bar{\pi})= \big(  \EE_{a_{m(h):h-1} \sim \mu^{\bar{\pi},h}}[g(z_h) \mid s_{m(h)}]\big)_{s_{m(h)} \in \Scal} \in \mathbb{R}^S,
	\end{aligned}
	\end{equation*}	
	completes the proof.
\end{proof}

%% file: appendix/app_linear.tex
The following lemma (akin to part (2) of \cref{lem:golf_property}) shows that $\Ecal^*$ has low rank structure with rank $\dlin$. The proof of \cref{thm:golf_linear_regret} is almost identical to proof of \cref{thm:golf_regret} where the only difference is to use \cref{lem:golf_property_linear} instead of part (2) of \cref{lem:golf_property} resulting in $S$ being replaced by $\dlin$ wherever it has been used.  
\begin{lemma}[akin to part (2) of \cref{lem:golf_property}]
\label{lem:golf_property_linear}
Under \cref{asm:linear}; for any policy $\pi$ and any function $f \in \Fcal$, and any $h \in [H]$, we have $\Ecal^{\star}_h(\pi,f) = \langle \zeta_h(\pi), \xi_h(f) \rangle$ where $\zeta_h(\pi) , \xi_h(f) \in \mathbb{R}^{\dlin}$ satisfy $\norm{\zeta_h(\pi)} \leq 1$ and $\norm{\xi_h(f)} \leq 2\sqrt{\dlin}$.
\end{lemma}
\begin{proof}[Proof of \cref{lem:golf_property_linear}]
	Let $g$ be a function $g:\Zcal_h \rightarrow [0,1]$ and $\tilde{\pi}^h = \pi \circ_{m(h)} \bar{\pi}$. 
	Recall that here $\tilde{\pi}^h$ is defined to take actions
    $a_{1:m(h)-1} \sim \pi$ and $a_{m(h):h-1}\sim \nu^{\bar{\pi},h}$ where
    $\pi$ and $\bar{\pi}$ may not be equal. Since $\mu^{\bar{\pi},h}$ is defined
    to be independent of the past given $s_{m(h)}$ we have the
    factorization
	\begin{align*}
	  \EE_{\tilde{\pi}^{h}}[g(z_h)] 
      &=  \EE_\pi \Big[\int_{s_{m(h)} \in \Scal} \langle \psi_\pi(s_{m(h)-1},a_{m(h)-1}) , \bmu(s_{m(h)})\cdot \EE_{a_{m(h):h-1} \sim \mu^{\bar{\pi},h}}[g(z_h) \mid s_{m(h)}] \Big]\\
      &= \big \langle \EE_\pi \psi_\pi(s_{m(h)-1},a_{m(h)-1}), \int_{s_{m(h)} \in \Scal}\bmu(s_{m(h)})\cdot \EE_{a_{m(h):h-1} \sim \mu^{\bar{\pi},h}}[g(z_h) \mid s_{m(h)}] \big \rangle
	\end{align*}
	We note that $\mu^{\bar{\pi},h}$ only depends on $(s,o)_{m(h):h-1}$ and $a_{m(h):h-2}$, thus the second term is independent of $\pi$ and only depends $g$ and $\bar{\pi}$. Define 
	\begin{equation*}
	\begin{aligned}
			&\zeta_h(\pi) \defeq \EE_\pi \psi_\pi(s_{m(h)-1},a_{m(h)-1}) \in \mathbb{R}^\dlin \quad \textrm{and} \quad \xi_h(g,\bar{\pi})= \int_{s_{m(h)} \in \Scal}\bmu(s_{m(h)})\cdot \EE_{a_{m(h):h-1} \sim \mu^{\bar{\pi},h}}[g(z_h) \mid s_{m(h)}] \in \mathbb{R}^\dlin.
	\end{aligned}
	\end{equation*}	
	Picking $g$ as $g(z_h) = (f_h-\Tcal_h f_{h+1})(z_h,\pi_f(z_h))$ and $\bar{\pi}= \pi_g$ completes the proof.
\end{proof}

%% file: appendix/app_H_ste_dec_pomdp.tex
\begin{algorithm*}[t]
\caption{\textbf{IS-RL}: Importance sampling for Reinforcement Learning} \label{alg:H step decodable POMDP}
 \begin{algorithmic}[1]
 \STATE \textbf{Initialize:}  $N$ number of samples, policy class $\Pi$,
 \STATE \textbf{Collect:} $N$ trajectories $\{o^{(t)}_h,a^{(t)}_h,r^{(t)}_h\}_{h=1}^H$ for $t \in [N]$ by executing the uniform policy $a^{(t)}_h \sim \unif(\Acal)$. 
 \STATE \label{line: is for pomdps}For any $\pi\in \Pi$ calculate its empirical value
 \begin{align*}
 \widehat{V}^\pi = \frac{1}{N}\sum_{t=1}^N \prod_{h=1}^H\rbr{\frac{\pi(a^{(t)}_h \mid z_h^{(t)})}{1/A}} \cdot \rbr{\sum_{h=1}^H r_h^{(t)}}
\end{align*}  
\STATE \textbf{Output}  $\widehat{\pi}\in \arg\max_{\pi\in \Pi} \widehat{V}^\pi.$
 \end{algorithmic}
\end{algorithm*}

In this section, we show that there exists an algorithm that returns
an $\epsilon$ optimal policy for any $H$-step decodable POMDP with
sample complexity which is only polynomial in $\abr{\Ocal}$, the
cardinality of the observation space. To do so, we construct a policy
class $\Pi$ that contains the optimal policy and has cardinality
bounded by $\abr{\Pi} \leq O\rbr{H(SA)^{2HSOA}}$ and we use this
policy class in a standard importance-sampling procedure. The
procedure is formally specified~\pref{alg:H step decodable POMDP},
and~\cref{prop: H step dec POMDP} follows immediately from
\cref{corr:policy_class_size} and \cref{lem:is_algo}.


\paragraph{Constructing the policy class $\Pi$ via recurrent function class.} 
Let $\Bcal_h$ denote the set of all mappings of the form $b_h:
\Scal_{h-1}\times \Acal_{h-1} \times \Ocal_h \rightarrow
\Scal_h$. This class represents all mappings from the latent state at
the previous time step, action at the previous time step, and current
observation to the latent state at the current time step. We call them
\emph{belief operators}.

We show that the latent state at time step $h$ is decodable from the
tuple $(o_h,s_{h-1},a_{h-1})$. In other words, we can write
$\phi^\star(z_{h}) =
b_h^\star(\phi^\star(z_{h-1}), a_{h-1},o_h)$ for some
belief operator $b_h^\star \in \Bcal_h$.
This relation is established in the following lemma.

\begin{lemma}\label{lem: decodability given o s a}
For each $h \in [H]$ there exists $b_h^\star \in \Bcal_h$ such
that for all reachable histories $z_h$ we have $\phi^\star(z_h) =
b_h^\star(\phi^\star(z_{h-1}), a_{h-1}, o_h)$.
\end{lemma}

Using the belief operator class we can design a policy class that
contains the optimal policy for any $H$-step decodable POMDP. Given a
decoder $\vec{b} := (b_1,\ldots,b_H) \in
\Bcal_1\times\ldots\times\Bcal_H$ and a trajectory $z_H$ (or a partial
trajectory $z_h$), the predicted state is updated recursively as
$\hat{s}_1 = b_1(o_1)$, $\hat{s}_h = b_h(\hat{s}_{h-1},
a_{h-1},o_h)$. Then we can define $\Pi_{\vec{b}} := \{\pi : \pi(a_h
\mid z_h) = \pi_h(a_h \mid \hat{s}_h)\}$, where here implicitly we are
updated $\hat{s}_h$ using $\vec{b}$. Then we can take $\Pi =
\bigcup_{\vec{b} \in \vec{\Bcal}} \Pi_b$. For this class we have the following corollary.

\begin{corollary}\label{corr:policy_class_size}
We have $|\Pi| \leq (SA)^{2SHOA}$ and for any $H$-step decodable POMDP
$\pi^\star \in \Pi$.
\end{corollary}

\paragraph{Importance Sampling Procedure for $H$-step POMDPs.} 
\pref{alg:H step decodable POMDP} describes a standard importance
sampling approach for policy learning in POMDPs, which is essentially
the same as the trajectory tree method
of~\citet{kearns1999approximate}. A standard analysis of importance
weighting using Bernstein's inequality and a uniform convergence
argument yield the following lemma. As the result is quite standard,
we omit the proof here.



\begin{lemma}\label{lem:is_algo}
Fix any $\epsilon,\delta>0$ and let $N = \Omega\rbr{HA^H
  \log\rbr{\abr{\Pi}/\delta}/\epsilon^2}$. Then with probability at
least $1-\delta$,~\pref{alg:H step decodable POMDP} returns a policy
$\widehat{\pi} \in \Pi$ such that
\begin{align*}
\max_{\pi\in \Pi} V^\pi \leq V^{\widehat{\pi}} +\epsilon.
\end{align*}
\end{lemma}

\subsection{Proofs}
We now turn to the proofs of \cref{lem: decodability given o s a} and \cref{corr:policy_class_size}.

\begin{proof}[Proof of~\cref{lem: decodability given o s a}]
By the decodability assumption, for any $z_h = (o_{1:h}, a_{1:h-1})$
such that $\sup_\pi \PP^\pi[z_h] > 0$, it holds that
\begin{align*}
\PP(s_h \mid z_h) = \delta\rbr{\phi^\star(z_h)}.
\end{align*}
On the other hand, it holds that
\begin{align}
\PP(s_h \mid z_h) = \frac{\sum_{s_{h-1}}\PP(s_h , o_h, s_{h-1} \mid o_{h-1:1}, a_{h-1:1})}{\sum_{s_{h-1}}\PP(o_h, s_{h-1} \mid o_{h-1:1}, a_{h-1:1})}. \label{eq: H step decodability}
\end{align}
By the POMDP model assumption and decodability the numerator is 
also given by,
\begin{align*}
\PP(s_h , o_h, s_{h-1} \mid o_{h-1:1}, a_{h-1:1})
 = \PP(s_h,o_h\mid s_{h-1},a_{h-1})\delta(s_{h-1} = \phi^\star(z_{h-1})).
\end{align*}
Similarly, the denominator is given by
\begin{align*}
\PP(o_h, s_{h-1} \mid o_{h-1:1}, a_{h-1:1}) \PP(s_h\mid s_{h-1},a_{h-1}) =
\sum_{\bar{s}_{h}} \PP(\bar{s}_h,o_h\mid s_{h-1},a_{h-1}) \delta(s_{h-1} = \phi^\star(z_{h-1})).
\end{align*}
Plugging this back into equation~\eqref{eq: H step decodability} we obtain
\begin{align}
\PP(s_h \mid z_h) \nonumber 
& = \frac{\sum_{s_{h-1}} \PP(s_h,o_h\mid s_{h-1},a_{h-1}) \delta(s_{h-1} = \phi^\star(z_{h-1}))}{\sum_{s_{h-1}} \sum_{\bar{s}_{h}} \PP(\bar{s}_h,o_h\mid s_{h-1},a_{h-1}) \delta(s_{h-1} = \phi^\star(z_{h-1}))} \nonumber \\
& = \frac{\PP(s_h,o_h\mid \phi^\star(z_{h-1}) ,a_{h-1})}{\sum_{\bar{s}_{h}} \PP(\bar{s}_h,o_h \mid \phi^\star(z_{h-1}),a_{h-1})} \nonumber \\
& = \frac{\PP(s_h\mid o_h,\phi^\star(z_{h-1}) ,a_{h-1})\PP(o_h \mid \phi^\star(z_{h-1}) ,a_{h-1})}{\sum_{\bar{s}_{h}} \PP(\bar{s}_h\mid \phi^\star(z_{h-1}),a_{h-1})\PP(o_h \mid \phi^\star(z_{h-1}) ,a_{h-1})} \nonumber \\
& = \frac{\PP(s_h\mid o_h,\phi^\star(z_{h-1}) ,a_{h-1})}{\sum_{\bar{s}_{h}} \PP(\bar{s}_h\mid \phi^\star(z_{h-1}),a_{h-1})} \nonumber \\
& = \PP(s_h\mid o_h,\phi^\star(z_{h-1}) ,a_{h-1}).\nonumber
\end{align}
Recall that $\PP(s_h \mid z_h) = \delta\rbr{s_h = \phi^\star(z_h)}$ by the decodability assumption. Hence, it holds that
\begin{align*}
\PP(s_h\mid o_h,\phi^\star(z_{h-1}) ,a_{h-1}) = \delta(s_h = \phi^\star(z_h)).
\end{align*}
Therefore for any reachable $z_h$, with $s_{h-1} =
\phi^\star(z_{h-1})$ we take $b_h^\star(s_{h-1},a_{h-1},o_{h})$ to be
the unique $s_h$ for which $\PP(s_h\mid o_h,s_{h-1},a_{h-1}) \ne 0$
and if this does not completely specify $b_h^\star$, we complete can
complete it arbitrarily.
\end{proof}

\begin{proof}[Proof of~\cref{corr:policy_class_size}]
The fact that $\pi^\star \in \Pi$ follows directly from \cref{lem:
  decodability given o s a}, since $\vec{b}^\star \in \Bcal$ and for
any H-step POMDP the optimal action depends only on the state. As for
the size of $\Pi$ observe that for each $h$ we have $|\Bcal_h|\leq
S^{SOA}$ and so $|\vec{\Bcal}| \leq S^{HSOA}$. Finally, for each
$\vec{b} \in \vec{\Bcal}$ we have $|\Pi_{\vec{b}}| = A^{SH}$. Taken
together we have $|\Pi|\leq (SA)^{HSOA}$ as desired. 
\end{proof}

%% file: appendix/app_brank.tex
Here we construct an instance of a $2$-step decodable POMDP in which
the bellmank rank scales with the number of observations $O$. We
further show that the OLIVE algorithm has sample complexity that
scales polynomially with $O$, thus motivating our new algorithmic
techniques. We believe a similer construction will also show that this
model does not fall into either the bilinear class or Bellman-Eluder
frameworks~\citep{du2021bilinear,jin2021bellman}.

The key idea is to use a construction inspired by the Hadamard
matrix. Let $O = 2^s$ for some natural number $s$ and $\Ocal = \{1,\ldots,O\}$. Then, there exist
sets $S_1, \ldots, S_{O-1} \subset \Ocal$ such that:
\begin{align}
\forall i: |S_i| = O/2, \quad \textrm{and} \quad \forall i \ne j: |S_i\cap S_j| = |S_i \cap \bar{S_j}| = O/4 \label{eq:set_system}
\end{align}
The existence of these can be verified by the existence and
orthogonality of Hadamard matrices in dimension $O = 2^s$. Indeed, if
we define $\{v_i\}_{i=0}^O \subset \{\pm 1\}^O$ such that $v_0 =
\mathbf{1}$ and $v_i$ is the $\pm 1$ indicator vector for set
$S_i$. Then the first property above is equivalent to $v_i^\top v_0 =
0$ for all $i \ne 0$ while the second property is equivalent to
\begin{align*}
  \forall i \ne j\in\{1,\ldots,O\}  \sum_{k}\one\{v_i[k] = +1\} v_j[k] = 0
\end{align*}
We claim that these two properties are satisfied if the vectors $v$
are the columns of a Hadamard matrix. The first follows directly from
orthogonality. For the second, since $v_i^\top v_j = 0$ and $v_j^\top v_0 = 0$ both by
orthogonality, we have
\begin{align*}
  v_i^\top v_j =0 \Rightarrow &\underbrace{\sum_{k}\one\{v_i[k] = +1\} v_j[k]}_{=:A_{ij}} - \underbrace{\sum_k \one\{v_i[k] = -1\} v_j[k]}_{=: B_{ij}} = 0\\
  v_j^\top v_0 = 0 \Rightarrow &\sum_{k}\one\{v_i[k] = +1\} v_j[k] + \sum_k \one\{v_i[k] = -1\} v_j[k] = 0.
\end{align*}
Thus we have $A_{ij} + B_{ij} = A_{ij} - B_{ij} = 0$ which implies
that $A_{ij} = 0$. So we have established the existence of $O-1$ sets
satisfying~\pref{eq:set_system}.

Let us now put this construction to use in a 2-step decodable
POMDP. We consider a $H=2$, three state POMDP with initial state $s_0$
and two states $s_1, s_2$ reachable at time $h=2$. We have: $\OO(\cdot
\mid s_0) = \textrm{Unif}(\{1,\ldots,O\})$ while $\OO(\cdot \mid s_1)
= \OO(\cdot \mid s_2) = \delta(\{\bot\})$. In words, from the initial
state we see an observation uniformly at random, while from $s_1$ or
$s_2$ we see no observation. The dynamics are such that taking $a_1$
from $s_0$ reaches $s_1$ and taking $a_2$ from $s_0$ reaches $s_2$. Only
a single action $a_1$ is available from $s_1$ or $s_2$ and it enjoys
reward $R(s_1,a_1) = 1/2$, $R(s_2,a_1) = 3/4$. Clearly this POMDP is
$2$-step decodable since the first state is always decodable and the
previous action uniquely determines the second state.

We have a function class $\Fcal$ of 2-step candidate $Q$
functions. The functions are $\Fcal := \{Q^\star\}\cup
\{f_i\}_{i=1}^{O-1}$ where each $f_i$ is associated with a set $S_i$
from the above Hadamard construction. These functions are defined as
\begin{align*}
f_i(oa_1) = \one\{o \in S_i\}, \quad f_i(oa_2) = 3/4, \quad f_i(oa_1\bot a_1) = \one\{o \in S_i\}, \quad f_i(oa_2\bot a_1) = 3/4
\end{align*}
It is easy to very that these functions have zero bellman error at the
first time step, that is
\begin{align*}
\forall (o,a): f_i(oa) = f_i(oa\bot a_1)
\end{align*}
On the other hand, $f_i$ has very high bellman error at the second
time step, since it never correctly predicts the reward for state
$s_1$. In particular we have $\EE_{d_2^{\pi_{f_i}}}[f_i(oa\bot a_1) -
  r] = 1/4$, since $\pi_{f_i}$ visits states $s_1$ on half
of the observations and every time it does it overpredicts the reward
by $1/2$. However, observe that
\begin{align*}
\EE_{d_2^{\pi_{f_i}}}[ f_j(oa\bot a_1) - r ] = \frac{1}{O} \sum_{o \in S_i} \one\{o \in S_j\}(1 - 1/2) + \one\{o \notin S_j\}(0 - 1/2) = 0,
\end{align*}
where the last identity uses~\pref{eq:set_system}. Thus we see that we
have embedded an $(O-1)\times (O-1)$-sized identity matrix inside of
the Bellman error matrix at time $2$, which shows that the Bellman
rank is $\Omega(O)$.

Note that the \olive~ algorithm itself will also incur $\textrm{poly}(O)$
sample complexity in this instance. This is because the value
predicted by $f_i$ at the starting state, namely $\EE[ \max_a f(oa)
]$, is $1/2 + 3/8$ which is greater than $V^\star=3/4$. Thus \olive~
will enumerate over the $f_i$ functions, eliminating one at a time and
incurring a $\textrm{poly}(O)$ sample complexity.

%% file: main.bbl
\begin{thebibliography}{42}
\providecommand{\natexlab}[1]{#1}
\providecommand{\url}[1]{\texttt{#1}}
\expandafter\ifx\csname urlstyle\endcsname\relax
  \providecommand{\doi}[1]{doi: #1}\else
  \providecommand{\doi}{doi: \begingroup \urlstyle{rm}\Url}\fi

\bibitem[Papadimitriou and Tsitsiklis(1987)]{papadimitriou1987complexity}
Christos~H Papadimitriou and John~N Tsitsiklis.
\newblock The complexity of markov decision processes.
\newblock \emph{Mathematics of operations research}, 12\penalty0 (3):\penalty0
  441--450, 1987.

\bibitem[Mossel and Roch(2005)]{mossel2005learning}
Elchanan Mossel and S{\'e}bastien Roch.
\newblock Learning nonsingular phylogenies and hidden markov models.
\newblock In \emph{Proceedings of the thirty-seventh annual ACM symposium on
  Theory of computing}, pages 366--375, 2005.

\bibitem[Jin et~al.(2020{\natexlab{a}})Jin, Kakade, Krishnamurthy, and
  Liu]{jin2020sample}
Chi Jin, Sham~M Kakade, Akshay Krishnamurthy, and Qinghua Liu.
\newblock Sample-efficient reinforcement learning of undercomplete pomdps.
\newblock \emph{arXiv:2006.12484}, 2020{\natexlab{a}}.

\bibitem[Mnih et~al.(2013)Mnih, Kavukcuoglu, Silver, Graves, Antonoglou,
  Wierstra, and Riedmiller]{mnih2013playing}
Volodymyr Mnih, Koray Kavukcuoglu, David Silver, Alex Graves, Ioannis
  Antonoglou, Daan Wierstra, and Martin Riedmiller.
\newblock Playing atari with deep reinforcement learning.
\newblock \emph{arXiv preprint arXiv:1312.5602}, 2013.

\bibitem[Mnih et~al.(2015)Mnih, Kavukcuoglu, Silver, Rusu, Veness, Bellemare,
  Graves, Riedmiller, Fidjeland, Ostrovski, et~al.]{mnih2015human}
Volodymyr Mnih, Koray Kavukcuoglu, David Silver, Andrei~A Rusu, Joel Veness,
  Marc~G Bellemare, Alex Graves, Martin Riedmiller, Andreas~K Fidjeland, Georg
  Ostrovski, et~al.
\newblock Human-level control through deep reinforcement learning.
\newblock \emph{nature}, 518\penalty0 (7540):\penalty0 529--533, 2015.

\bibitem[Hessel et~al.(2018)Hessel, Modayil, Van~Hasselt, Schaul, Ostrovski,
  Dabney, Horgan, Piot, Azar, and Silver]{hessel2018rainbow}
Matteo Hessel, Joseph Modayil, Hado Van~Hasselt, Tom Schaul, Georg Ostrovski,
  Will Dabney, Dan Horgan, Bilal Piot, Mohammad Azar, and David Silver.
\newblock Rainbow: Combining improvements in deep reinforcement learning.
\newblock In \emph{Thirty-second AAAI conference on artificial intelligence},
  2018.

\bibitem[Jin et~al.(2021)Jin, Liu, and Miryoosefi]{jin2021bellman}
Chi Jin, Qinghua Liu, and Sobhan Miryoosefi.
\newblock Bellman eluder dimension: New rich classes of rl problems, and
  sample-efficient algorithms.
\newblock \emph{arXiv preprint arXiv:2102.00815}, 2021.

\bibitem[Jiang et~al.(2017)Jiang, Krishnamurthy, Agarwal, Langford, and
  Schapire]{jiang2017contextual}
Nan Jiang, Akshay Krishnamurthy, Alekh Agarwal, John Langford, and Robert~E
  Schapire.
\newblock Contextual decision processes with low bellman rank are
  pac-learnable.
\newblock In \emph{International Conference on Machine Learning}, pages
  1704--1713. PMLR, 2017.

\bibitem[Du et~al.(2021)Du, Kakade, Lee, Lovett, Mahajan, Sun, and
  Wang]{du2021bilinear}
Simon~S Du, Sham~M Kakade, Jason~D Lee, Shachar Lovett, Gaurav Mahajan, Wen
  Sun, and Ruosong Wang.
\newblock Bilinear classes: A structural framework for provable generalization
  in rl.
\newblock \emph{arXiv preprint arXiv:2103.10897}, 2021.

\bibitem[Hausknecht and Stone(2015)]{hausknecht2015deep}
Matthew Hausknecht and Peter Stone.
\newblock Deep recurrent q-learning for partially observable mdps.
\newblock In \emph{2015 aaai fall symposium series}, 2015.

\bibitem[Zhu et~al.(2017)Zhu, Li, Poupart, and Miao]{zhu2017improving}
Pengfei Zhu, Xin Li, Pascal Poupart, and Guanghui Miao.
\newblock On improving deep reinforcement learning for pomdps.
\newblock \emph{arXiv preprint arXiv:1704.07978}, 2017.

\bibitem[Igl et~al.(2018)Igl, Zintgraf, Le, Wood, and Whiteson]{igl2018deep}
Maximilian Igl, Luisa Zintgraf, Tuan~Anh Le, Frank Wood, and Shimon Whiteson.
\newblock Deep variational reinforcement learning for pomdps.
\newblock In \emph{International Conference on Machine Learning}, pages
  2117--2126. PMLR, 2018.

\bibitem[Hafner et~al.(2019)Hafner, Lillicrap, Ba, and
  Norouzi]{hafner2019dream}
Danijar Hafner, Timothy Lillicrap, Jimmy Ba, and Mohammad Norouzi.
\newblock Dream to control: Learning behaviors by latent imagination.
\newblock \emph{arXiv preprint arXiv:1912.01603}, 2019.

\bibitem[McCallum(1993)]{mccallum1993overcoming}
R~Andrew McCallum.
\newblock Overcoming incomplete perception with utile distinction memory.
\newblock In \emph{Proceedings of the Tenth International Conference on Machine
  Learning}, pages 190--196, 1993.

\bibitem[Kearns et~al.(1999)Kearns, Mansour, and Ng]{kearns1999approximate}
Michael~J Kearns, Yishay Mansour, and Andrew~Y Ng.
\newblock Approximate planning in large pomdps via reusable trajectories.
\newblock In \emph{NIPS}, pages 1001--1007. Citeseer, 1999.

\bibitem[Kearns et~al.(2002)Kearns, Mansour, and Ng]{kearns2002sparse}
Michael Kearns, Yishay Mansour, and Andrew~Y Ng.
\newblock A sparse sampling algorithm for near-optimal planning in large markov
  decision processes.
\newblock \emph{Machine learning}, 49\penalty0 (2):\penalty0 193--208, 2002.

\bibitem[Even-Dar et~al.(2005)Even-Dar, Kakade, and
  Mansour]{evendar2005reinforcement}
Eyal Even-Dar, Sham~M. Kakade, and Yishay Mansour.
\newblock Reinforcement learning in pomdps without resets.
\newblock In \emph{International Joint Conference on Artificial Intelligence},
  2005.

\bibitem[Azizzadenesheli et~al.(2016)Azizzadenesheli, Lazaric, and
  Anandkumar]{azizzadenesheli2016reinforcement}
Kamyar Azizzadenesheli, Alessandro Lazaric, and Animashree Anandkumar.
\newblock Reinforcement learning of pomdps using spectral methods.
\newblock In \emph{Conference on Learning Theory}, pages 193--256. PMLR, 2016.

\bibitem[Guo et~al.(2016)Guo, Doroudi, and Brunskill]{guo2016pac}
Zhaohan~Daniel Guo, Shayan Doroudi, and Emma Brunskill.
\newblock A pac rl algorithm for episodic pomdps.
\newblock In \emph{Artificial Intelligence and Statistics}, pages 510--518.
  PMLR, 2016.

\bibitem[Anandkumar et~al.(2014)Anandkumar, Ge, Hsu, Kakade, and
  Telgarsky]{anandkumar2014tensor}
Animashree Anandkumar, Rong Ge, Daniel Hsu, Sham~M Kakade, and Matus Telgarsky.
\newblock Tensor decompositions for learning latent variable models.
\newblock \emph{Journal of machine learning research}, 15:\penalty0 2773--2832,
  2014.

\bibitem[Sun et~al.(2019)Sun, Jiang, Krishnamurthy, Agarwal, and
  Langford]{sun2019model}
Wen Sun, Nan Jiang, Akshay Krishnamurthy, Alekh Agarwal, and John Langford.
\newblock Model-based rl in contextual decision processes: Pac bounds and
  exponential improvements over model-free approaches.
\newblock In \emph{Conference on learning theory}, pages 2898--2933. PMLR,
  2019.

\bibitem[Foster et~al.(2021)Foster, Kakade, Qian, and
  Rakhlin]{foster2021statistical}
Dylan~J Foster, Sham~M Kakade, Jian Qian, and Alexander Rakhlin.
\newblock The statistical complexity of interactive decision making.
\newblock \emph{arXiv preprint arXiv:2112.13487}, 2021.

\bibitem[Du et~al.(2019)Du, Krishnamurthy, Jiang, Agarwal, Dudik, and
  Langford]{du2019provably}
Simon Du, Akshay Krishnamurthy, Nan Jiang, Alekh Agarwal, Miroslav Dudik, and
  John Langford.
\newblock Provably efficient rl with rich observations via latent state
  decoding.
\newblock In \emph{International Conference on Machine Learning}, pages
  1665--1674. PMLR, 2019.

\bibitem[Misra et~al.(2020)Misra, Henaff, Krishnamurthy, and
  Langford]{misra2020kinematic}
Dipendra Misra, Mikael Henaff, Akshay Krishnamurthy, and John Langford.
\newblock Kinematic state abstraction and provably efficient rich-observation
  reinforcement learning.
\newblock In \emph{International conference on machine learning}, pages
  6961--6971. PMLR, 2020.

\bibitem[Agarwal et~al.(2020)Agarwal, Kakade, Krishnamurthy, and
  Sun]{agarwal2020flambe}
Alekh Agarwal, Sham Kakade, Akshay Krishnamurthy, and Wen Sun.
\newblock Flambe: Structural complexity and representation learning of low rank
  mdps.
\newblock \emph{Advances in Neural Information Processing Systems}, 33, 2020.

\bibitem[Uehara et~al.(2021)Uehara, Zhang, and Sun]{uehara2021representation}
Masatoshi Uehara, Xuezhou Zhang, and Wen Sun.
\newblock Representation learning for online and offline rl in low-rank mdps.
\newblock \emph{arXiv preprint arXiv:2110.04652}, 2021.

\bibitem[Ljung(1998)]{ljung1998system}
Lennart Ljung.
\newblock \emph{System Identification: Theory for the User}.
\newblock Pearson Education, 1998.

\bibitem[Box et~al.(2015)Box, Jenkins, Reinsel, and Ljung]{box2015time}
George~EP Box, Gwilym~M Jenkins, Gregory~C Reinsel, and Greta~M Ljung.
\newblock \emph{Time series analysis: forecasting and control}.
\newblock John Wiley \& Sons, 2015.

\bibitem[Hamilton(1994)]{hamilton1994time}
James~Douglas Hamilton.
\newblock \emph{Time series analysis}.
\newblock Princeton university press, 1994.

\bibitem[Kalman(1960)]{kalman1960new}
Rudolph~Emil Kalman.
\newblock A new approach to linear filtering and prediction problems.
\newblock \emph{Journal of Basic Engineering}, 1960.

\bibitem[Verhaegen(1993)]{verhaegen1993subspace}
Michel Verhaegen.
\newblock Subspace model identification part 3. analysis of the ordinary
  output-error state-space model identification algorithm.
\newblock \emph{International Journal of control}, 58\penalty0 (3):\penalty0
  555--586, 1993.

\bibitem[Arora et~al.(2018)Arora, Hazan, Lee, Singh, Zhang, and
  Zhang]{arora2018towards}
Sanjeev Arora, Elad Hazan, Holden Lee, Karan Singh, Cyril Zhang, and Yi~Zhang.
\newblock Towards provable control for unknown linear dynamical systems.
\newblock In \emph{International Conference on Learning Representations,
  Workshop Track}, 2018.

\bibitem[Agarwal et~al.(2019)Agarwal, Bullins, Hazan, Kakade, and
  Singh]{agarwal2019online}
Naman Agarwal, Brian Bullins, Elad Hazan, Sham Kakade, and Karan Singh.
\newblock Online control with adversarial disturbances.
\newblock In \emph{International Conference on Machine Learning}, pages
  111--119. PMLR, 2019.

\bibitem[Oymak and Ozay(2019)]{oymak2019non}
Samet Oymak and Necmiye Ozay.
\newblock Non-asymptotic identification of lti systems from a single
  trajectory.
\newblock In \emph{2019 American control conference (ACC)}, pages 5655--5661.
  IEEE, 2019.

\bibitem[Simchowitz et~al.(2019)Simchowitz, Boczar, and
  Recht]{simchowitz2019learning}
Max Simchowitz, Ross Boczar, and Benjamin Recht.
\newblock Learning linear dynamical systems with semi-parametric least squares.
\newblock In \emph{Conference on Learning Theory}, pages 2714--2802. PMLR,
  2019.

\bibitem[Krishnamurthy et~al.(2016)Krishnamurthy, Agarwal, and
  Langford]{krish2016hardness}
Akshay Krishnamurthy, Alekh Agarwal, and John Langford.
\newblock Pac reinforcement learning with rich observations.
\newblock In D.~Lee, M.~Sugiyama, U.~Luxburg, I.~Guyon, and R.~Garnett,
  editors, \emph{Advances in Neural Information Processing Systems}, volume~29.
  Curran Associates, Inc., 2016.
\newblock URL
  \url{https://proceedings.neurips.cc/paper/2016/file/2387337ba1e0b0249ba90f55b2ba2521-Paper.pdf}.

\bibitem[Weisz et~al.(2021)Weisz, Amortila, and
  Szepesv{\'a}ri]{weisz2021exponential}
Gell{\'e}rt Weisz, Philip Amortila, and Csaba Szepesv{\'a}ri.
\newblock Exponential lower bounds for planning in mdps with
  linearly-realizable optimal action-value functions.
\newblock In \emph{Algorithmic Learning Theory}, pages 1237--1264. PMLR, 2021.

\bibitem[Antos et~al.(2008)Antos, Szepesv{\'a}ri, and Munos]{antos2008learning}
Andr{\'a}s Antos, Csaba Szepesv{\'a}ri, and R{\'e}mi Munos.
\newblock Learning near-optimal policies with bellman-residual minimization
  based fitted policy iteration and a single sample path.
\newblock \emph{Machine Learning}, 71\penalty0 (1):\penalty0 89--129, 2008.

\bibitem[Chen and Jiang(2019)]{chen2019info}
Jinglin Chen and Nan Jiang.
\newblock Information-theoretic considerations in batch reinforcement learning.
\newblock In Kamalika Chaudhuri and Ruslan Salakhutdinov, editors,
  \emph{Proceedings of the 36th International Conference on Machine Learning},
  volume~97 of \emph{Proceedings of Machine Learning Research}, pages
  1042--1051. PMLR, 09--15 Jun 2019.
\newblock URL \url{https://proceedings.mlr.press/v97/chen19e.html}.

\bibitem[Azar et~al.(2017)Azar, Osband, and Munos]{azar2017minimax}
Mohammad~Gheshlaghi Azar, Ian Osband, and R{\'e}mi Munos.
\newblock Minimax regret bounds for reinforcement learning.
\newblock In \emph{International Conference on Machine Learning}, pages
  263--272. PMLR, 2017.

\bibitem[Jin et~al.(2020{\natexlab{b}})Jin, Yang, Wang, and
  Jordan]{jin2020provably}
Chi Jin, Zhuoran Yang, Zhaoran Wang, and Michael~I Jordan.
\newblock Provably efficient reinforcement learning with linear function
  approximation.
\newblock In \emph{Conference on Learning Theory}, pages 2137--2143. PMLR,
  2020{\natexlab{b}}.

\bibitem[Russo and Van~Roy(2013)]{russo2013eluder}
Daniel Russo and Benjamin Van~Roy.
\newblock Eluder dimension and the sample complexity of optimistic exploration.
\newblock In C.~J.~C. Burges, L.~Bottou, M.~Welling, Z.~Ghahramani, and K.~Q.
  Weinberger, editors, \emph{Advances in Neural Information Processing
  Systems}, volume~26. Curran Associates, Inc., 2013.
\newblock URL
  \url{https://proceedings.neurips.cc/paper/2013/file/41bfd20a38bb1b0bec75acf0845530a7-Paper.pdf}.

\end{thebibliography}
